\documentclass[onefignum,onetabnum]{siamart190516}

\usepackage{fullpage}
\usepackage{epsfig}
\usepackage{graphicx}
\usepackage{amsmath}
\usepackage{amssymb}
\usepackage{color}
\usepackage{pdflscape}
\usepackage{algorithm,algpseudocode}%
\usepackage{booktabs}
\usepackage{comment}
\usepackage{threeparttable}
\usepackage{cancel}
\usepackage{hyperref}

\usepackage{amssymb}
\usepackage{pifont}%

\usepackage[colorinlistoftodos,bordercolor=orange,backgroundcolor=orange!20,linecolor=orange,textsize=scriptsize]{todonotes}

\newtheorem{thm}{Theorem}

\newtheorem{lem}[thm]{Lemma}

\newtheorem{cor}[thm]{Corollary}

\newtheorem{dfn}[thm]{Definition}

\newtheorem{rem}{Remark}

\newcommand{\inner}[1]{\left\langle #1\right\rangle}

\newcommand{\R}{\mathbb{R}}

\algnewcommand{\Inputs}{%
	\State \textbf{Inputs:}
}
\algnewcommand{\Initialize}{%
	\State \textbf{Initialize:}
}
\algnewcommand{\Outputs}{%
	\State \textbf{Outputs:}
}
\algnewcommand{\ForLoop}{%
	\textbf{For $j=1,2,...,m$ do}
}

\algnewcommand{\ForOuterLoop}{%
	\textbf{For $k=1,2,..., K$ do}
}

\algnewcommand{\ForLoopKZPT}{%
	\textbf{For $j=1,2,..., p \lfloor {m \over p}\rfloor $ do}
}

\algnewcommand{\ForLoopmod}{%
	\textbf{For $L=0,1,...,L-1$ do}
}
\algnewcommand{\ForEnd}{%
	\textbf{End For}
}
\algnewcommand{\Iterate}{%
	\State \textbf{Iterate:}
}

\allowdisplaybreaks

\begin{document}

\title{Linear Convergence of Reshuffling Kaczmarz Methods With Sparse Constraints}

\author{
Halyun Jeong\thanks{University of California Los Angeles, Department of Mathematics, Los Angeles, CA 90095
  (\email{hajeong@math.ucla.edu})} \and Deanna Needell\thanks{University of California Los Angeles, Department of Mathematics, Los Angeles, CA 90095
  (\email{deanna@math.ucla.edu})}}
\maketitle

\begin{abstract}
The Kaczmarz method (KZ) and its variants, which are types of stochastic gradient descent (SGD) methods, have been extensively studied due to their simplicity and efficiency in solving linear equation systems. The iterative thresholding  (IHT) method has gained popularity in various research fields, including compressed sensing or sparse linear regression, machine learning with additional structure, and optimization with nonconvex constraints. 
Recently, a hybrid method called Kaczmarz-based IHT (KZIHT) has been proposed, combining the benefits of both approaches, but its theoretical guarantees are missing.   
In this paper, we provide the first theoretical convergence guarantees for KZIHT by showing that it converges linearly to the solution of a system with sparsity constraints up to optimal statistical bias when the reshuffling data sampling scheme is used. 
	
We also propose the Kaczmarz with periodic thresholding (KZPT) method, which generalizes KZIHT by applying the thresholding operation for every certain number of KZ iterations and by employing two different types of step sizes.
We establish a linear convergence guarantee for KZPT for randomly subsampled bounded orthonormal systems (BOS) and mean-zero isotropic sub-Gaussian random matrices, which are most commonly used models in compressed sensing, dimension reduction, matrix sketching, and many inverse problems in neural networks. Our analysis shows that KZPT with an optimal thresholding period outperforms KZIHT. To support our theory, we include several numerical experiments.
			
	\end{abstract}


\section{Introduction}
The Kaczmarz method (KZ) is an iterative method for solving systems of linear equations, attracting considerable interest due to its computational efficiency and simplicity \cite{karczmarz1937angenaherte}. KZ has also played a prominent role in image and signal processing under the name of \emph{algebraic reconstruction technique} (ART) \cite{gordon1970algebraic}. 

To solve a linear system of the form $Ax = b$, KZ iteratively selects a row $a_i^T$ of $A$ according to some selection rule and orthogonally projects the iterate onto the solution space of $a_i^Tx = b$.  Various row selection rules can be employed, such as cyclic \cite{tanabe1971projection, censor1981row}, greedy \cite{eldar2011acceleration, griebel2012greedy}, or random \cite{strohmer2009randomized, needell2014paved, chen2012almost}. Shortly after the foundational theoretical work by Strohmer and Vershynin \cite{strohmer2009randomized} showing a linear convergence for the randomized Kaczmarz method, several  refinements and extensions of KZ have been studied \cite{needell2010randomized, needell2014paved, ma2015convergence, nutini2016convergence}. 

The Kaczmarz method can be viewed as a special case of a stochastic gradient descent (SGD) algorithm for the squared loss function \cite{needell2014stochastic}. SGD often provides significant advantages
over traditional gradient descent because it does not require the full gradient computation, which could be challenging in many scenarios. Consequently, SGD has become a standard approach for training machine learning models, and there have been a growing number of works that propose and analyze its variants. One of the important development in this direction is the proximal SGD, which is an SGD followed by a proximal step to solve structured machine learning problems. Another active area of SGD research, when the stochastic gradient is associated with each data point, is studying the effect of sampling without replacement (reshuffling). Reshuffling-based stochastic methods have recently gained significant attention from practitioners because of their efficiency over the traditional SGD in training machine learning models \cite{ahn2020sgd, mishchenko2020random, sun2020optimization, gurbuzbalaban2021random}. Notably, the work \cite{mishchenko2022proximal} combines the idea of proximal SGD  with reshuffling and introduces proximal reshuffling (ProxRR), a promising stochastic method based on regularization for constrained optimization.
In particular, when the objective function is the squared loss function and the proximal operator is applied to the $\ell_1$ regularizer, their approach becomes the reshuffling Kaczmarz with soft-thresholding applied at every epoch. 

In a different line of research, \emph{iterative hard thresholding} (IHT) has become a standard approach in modern signal processing, which is a gradient descent algorithm followed by projection to the signal set of interest. Numerous works on IHT-based methods have appeared in compressed sensing \cite{blumensath2009iterative, nguyen2017linear}, low-rank matrix recovery \cite{jain2010guaranteed, goldfarb2011convergence}, tensor recovery \cite{rauhut2017low, grotheer2021iterative}, and inverse problems in neural networks \cite{shah2018solving, jagatap2019algorithmic}.

Recently, Zhang et al. \cite{zhang2015iterative} have proposed an accelerated version of IHT based on the Kaczmarz method (KZIHT) for compressed sensing by replacing the gradient step in IHT with KZ iterations --- making an interesting connection between two seemingly unrelated approaches. KZIHT is now considered to be one of the state-of-the-art methods for compressed sensing.  While the authors of \cite{zhang2015iterative} empirically demonstrate its computational benefits over IHT, the convergence guarantees of KZIHT are still missing. 

In this paper, to our best knowledge, we present the first proof of the linear convergence of KZIHT with reshuffling. We also propose Kaczmarz with periodic thresholding (KZPT), a new method  generalizing KZIHT  by applying the thresholding operation periodically for every certain KZ iterations and by employing two different types of step sizes. 
Then, we show that KZPT enjoys a linear convergence as well, and with the proper choice of thresholding period, it can outperform KZIHT.

It turns out that KZIHT with reshuffling can be interpreted as a counterpart to ProxRR in \cite{mishchenko2022proximal} but with hard-thresholding applied per epoch instead of soft-thresholding. However, unlike our work, the convergence analysis of ProxRR in \cite{mishchenko2022proximal} is not directly applicable to the sparse linear regression problem as it requires regularizers to be strongly convex. 
     
Our linear convergence guarantees are given for randomly subsampled bounded orthonormal systems (BOS) and random matrices whose rows are independent, mean-zero, isotropic sub-Gaussian. These two are the most commonly used models in compressed sensing, dimension reduction, matrix sketching, and inverse problems in neural networks \cite{foucart2013invitation, dirksen2016dimensionality, derezinski2021sparse, shah2018solving,jagatap2019algorithmic}. 
The sub-Gaussian model includes important random matrix types such as Bernoulli matrices and Gaussian matrices.
	
Furthermore, in the noisy measurement setting, our analysis reveals that KZIHT/KZPT offer recovery guarantees up to the optimal statistical bias for the sparse regression problem.  

A key aspect of our theoretical findings involves a careful analysis of a product of certain types of random matrices that arise from various stochastic algorithms \cite{recht2012toward, huang2021matrix}, which could be of independent interest.

\subsection{Organization}
    
The rest of the paper has the following structure. In Section 2,  after giving the background of KZ, IHT, and KZIHT, we introduce our main method, KZPT. This method generalizes KZIHT by applying the thresholding operation for every certain number of KZ iterations, referred to as a thresholding period or window, and by employing two different types of step sizes. 
    
Section 3 contains our main convergence result for KZIHT and its proof starting with the multi-step analysis of KZ. Then, we extend our analysis to show a linear convergence of KZPT up to the optimal statistical bias. In Section 4, we explain how to optimize the thresholding period to achieve the optimal convergence rate. Our results indicate that with a proper choice of period, KZPT outperforms KZIHT or IHT. Section 5 includes numerical experiments to complement our theory. We discuss the other related works in Section 6 and summarize our work in Section 7.

\subsection{Notation}
\label{subsection:Notation}
A vector is called $s$-sparse if it has at most $s$ nonzero entries. 
Let us denote by $T_s$ the hard thresholding operator onto the set of all $s$-sparse vectors, zeroing out all the entries of a vector but keeping the $s$-largest magnitude elements. The $\ell_2$ norm for a vector is denoted by either $\|\cdot\|$ or $\|\cdot\|_2$, and $\|\cdot\|_0$ denotes the $\ell_0$ ``norm" that counts the number of nonzero components of a vector. The operator norm and Frobenius norm of a matrix are denoted by $\|\cdot\|$ and $\|\cdot\|_F$ respectively.
For a random variable $X$, we define the sub-Gaussian norm $\|X\|_{\psi_2}$ of $X$ as 
$\|X\|_{\psi_2} = \inf \{t > 0: \mathbb{E}(X^2/t^2) \le 2\}$. We say a random variable sub-Gaussian if its sub-Gaussian norm is finite. 
For a random vector $u \in \R^N$, its sub-Gaussian norm is defined as $\sup\limits_{w \in \mathbb{S}^{N-1} } \| \inner{w, u}\|_{\psi_2}$ where $\mathbb{S}^{N-1}$ is the unit sphere in $\R^N$. We say a random vector $u \in \R^N$ is isotropic if $\mathbb{E} uu^T = I_{N \times N}$ where $I_{N \times N}$ is the $N \times N$ identity matrix.
For a given positive integer $m$, $[m]$ denotes the set of integers $\{1, 2, \dots, m\}$.

\section{Kaczmarz with periodic thresholding}
	
In this section, we briefly overview IHT and KZIHT, and introduce our approach KZPT.
The goal of compressed sensing is to recover an $s$-sparse vector $x$ in $\R^N$ from an underdetermined system $b = Ax + e$, where $e$ is the noise vector. Here $A$ is an $m \times N$ sampling matrix and its $j$-th row is denoted by $a_j^T$. 
\emph{Iterative hard thresholding }(IHT) is a gradient descent method with hard thresholding that aims to solve the following optimization problem for compressed sensing \cite{blumensath2009iterative}:
\begin{equation} \label{prob:IHT_optimization} 
	\min_{z \in \R^N} \|Az - b\|_2^2  \quad
		\text{subject to} \quad \|z\|_0 \le s.
	\end{equation} The description of IHT is given in Algorithm \ref{alg:IHT} using the hard thresholding operator $T_s$ defined in Section \ref{subsection:Notation}

	\begin{algorithm} [ht]
	\caption{IHT }
	\begin{algorithmic}
		\Inputs $m\times N$ sensing matrix $A$, measurements $b=Ax + e$, sparsity level $s$, the number of iterations $K$
		
		\Initialize	$x^0 = 0$. 

             \ForOuterLoop
	
		\indent \indent \textbf{Compute $z^{k+1}$}
		\[
		z^{k+1} := x^k + {1 \over m} A^T(b - Ax^k) 
		\]
		
		\indent  \indent \textbf{Project to the $s$-sparse vector set}\\

		\[
		x^{k+1} := T_s(z^{k+1}) 
		\]

        \indent \ForEnd
        
		\Outputs The approximate $s$-sparse solution to $Ax=b$.
		
	\end{algorithmic}
	\label{alg:IHT}
\end{algorithm}
	
On the other hand, the Kaczmarz method is an iterative procedure to solve $Ax = b$ by running the following iteration
	\begin{align}
	    \label{eq:Standard_KZ_iteration}
	    x_{j+1} := x_j + {(b_{\tau(j)} - a^T_{\tau(j)}  x_j ) \over \|a_{\tau(j)}\|_2^2} a_{\tau(j)},
	\end{align}
where $\tau(j)$ is some row selection rule governing the selection of the rows of $A$ in iteration $j$.
It is easy to see that at each iteration $j$, KZ orthogonally projects $x_j$ to the solution set of $a^T_{\tau(j)}x = b_{\tau(j)}$. The main novelty of KZIHT is in replacing the gradient step in IHT with $m$ number of KZ iterations in \eqref{eq:Standard_KZ_iteration}. More precisely, as described in Algorithm \ref{alg:KZIHT}, in each iteration $k$, we select indices $\tau(1), \tau(2), \dots, \tau(m)$ according to the row selection rule $\tau$ and run $m$ number of KZ iterations 
\[
        x^k_{j+1} := x^k_j + \gamma {(b_{\tau(j)} - a^T_{\tau(j)} x^k_j ) \over \|a_{\tau(j)}\|_2^2} a_{\tau(j)}.
 \] After this, we apply the hard thresholding operator $T_s$ to $x^k_{m+1}$ and set $x^{k+1}_1 := T_s(x^k_{m+1})$. 
    
In the pseudocode of Algorithm \ref{alg:KZIHT}, $x^k_j$ denotes the $k$-th iterate of the algorithm with $j$-th KZ (inner) iteration. We are going to use the same notation for iterates for other KZ-based algorithms in the paper, the superscript denotes the main iteration index, and the subscript corresponds to the KZ iteration index. 
 
Note that the KZ iterations in Algorithm \ref{alg:KZIHT}  are slightly more general than the ones in \cite{zhang2015iterative} by employing a step size $\gamma$.

\begin{algorithm} 
    \caption{KZIHT}
    \begin{algorithmic}
		\Inputs $m\times N$ measurement matrix $A$, measurements $b=Ax + e$, sparsity level $s$, step size $\gamma$, the number of iterations $K$
			
			\Initialize $x^{1}_1 = 0$.

            \ForOuterLoop
			
	\indent \indent \ForLoop 
			\[
				x^k_{j+1} := x^k_j + \gamma {(b_{\tau(j)} - a^T_{\tau(j)} x^k_j ) \over \|a_{\tau(j)}\|_2^2} a_{\tau(j)} \qquad \text{[Kaczmarz loop with step size $\gamma$]}
			\] 
	\indent \indent \ForEnd
			
	\indent \indent \textbf{Project $x^k_{m+1}$ to $s$-sparse vector sets}\\
			\[
			x^{k+1}_1 := T_s(x^k_{m+1}) \qquad \text{[Hard thresholding]}
			\]

    \indent \ForEnd
            
    \Outputs The approximate $s$-sparse solution to $Ax=b$.
			
    \end{algorithmic}
    \label{alg:KZIHT}
\end{algorithm}

\begin{rem}
In \cite{zhang2015iterative}, KZIHT is called KZIMT and the method that is denoted there as KZIHT actually refers to a different method. This unfortunate naming conflict arises because they call the hard thresholding operator $T_s$ by a different name. 
\end{rem}
	
We now present our method, Kaczmarz with periodic thresholding (KZPT). KZPT generalizes KZIHT in two ways: by introducing the gradient step size $\lambda$ and by applying the hard thresholding operation periodically with thresholding period $p$, as described in Algorithm \ref{alg:KZPT}. One can immediately verify that when the parameters $\lambda = 1$ and the period is $m$, KZPT coincides with KZIHT. 
	
	\begin{algorithm}
		\caption{KZPT}
		\begin{algorithmic}
			\Inputs $m\times N$ measurement matrix $A$, measurements $b=Ax +e$, sparsity level $s$, step size $\gamma$ for Kaczmarz inner iteration, step size $\lambda$ for IHT, hard thresholding period $p$, the number of iterations $K$
			
	    	\Initialize $x^{1}_1 = 0$.
			
             \ForOuterLoop
			
		\indent \indent	\ForLoopKZPT
			\begin{align*}
			 \quad	x^k_{j+1} := x^k_j + \gamma {(b_{\tau(j)} - a^T_{\tau(j)} x^k_j ) \over \|a_{\tau(j)}\|_2^2} a_{\tau(j)} \qquad \text{[Kaczmarz loop with step size $\gamma$]}	
			\end{align*}
		
		\indent	\indent  \indent \textbf{If $j \equiv 0 \pmod p$, apply IHT gradient and project $x^k_j$ to $s$-sparse vector sets}\\
			
			\begin{align*}
				&\qquad g^k := x^k_{j+1} - x^k_1\\
				&\qquad u^k := x^k_1 + \lambda g^k \qquad  \text{[IHT gradient step with step size $\lambda$]}
			\end{align*}
			\[
			x^k_{j+1} := T_s(u^k) \qquad \quad \text{[Peroidic hard thresholding]}
			\]
						
			\indent \indent \ForEnd
			\begin{align*}
				&x^{k+1}_1 = x^k_{p \lfloor {m \over p}\rfloor +1} \qquad \qquad \qquad \qquad \qquad \qquad \qquad \qquad \qquad \qquad \qquad \qquad \qquad
			\end{align*}
				
			 \ForEnd
    
			\Outputs The approximate $s$-sparse solution to $Ax=b$.
			
		\end{algorithmic}
		\label{alg:KZPT}
	\end{algorithm}

	\section{Linear Convergence of KZPT and KZIHT}
	\label{section:KZPT_period_1}
	This section presents our proof of a linear convergence of KZPT for two commonly used sensing matrices, randomly subsampled BOS and mean-zero isotropic sub-Gaussian random matrices. We begin with KZPT with period $p = m$, a simpler setting to illustrate the idea of our analysis. 	
	For the convenience of our readers, we include KZPT for this case in Algorithm \ref{alg:KZPT_period_m}.
	The first step of our analysis is the following lemma about the equation for multi-step KZ iterations. 
	\begin{lem}
	   \label{lem:multi_step_equation}
	    Let $x$ be an $s$-sparse vector with $b= Ax + e$, where $e$ is the noise vector. We denote by $x^k_j$ the $j$-th iterate of KZ iterations at $k$-th step of KZPT. Then, the approximation error of KZPT satisfies the following identity. 
	\begin{align}
	    \label{eq:main_equation_in_product_form}
	    x^k_{m+1} - x 
		&=  \left(I - \gamma {a_{\tau(m)} a_{\tau(m)}^T \over \|a_{\tau(m)}\|_2^2} \right) \left(I - \gamma {a_{\tau({m-1})} a_{\tau({m-1})}^T \over \|a_{\tau({m-1})}\|_2^2} \right) \cdots \left(I - \gamma {a_{\tau(1)} a_{\tau(1)}^T \over \|a_{\tau(1)}\|_2^2} \right) (x^k_1 -x)\\
		\nonumber
		& \qquad + \gamma \sum_{i=0}^{m-1} e_{\tau(m-i)} \prod_{j=0}^{i-1}  \left(I - \gamma {a_{\tau({m-j})} a_{\tau({m-j})}^T \over \|a_{\tau({m-j})}\|_2^2} \right) {a_{\tau(m-i)} \over \|a_{\tau(m-i)}\|_2^2},
	\end{align}
	 with the convention $\prod_{p}^q = 1$ if $p > q$.
	\end{lem}

	\begin{proof} [Proof of Lemma \ref{lem:multi_step_equation}]
	 By subtracting $x$ from the KZ iteration equation in Algorithm \ref{alg:KZPT_period_m}, we have
	\begin{align*}
		&x^k_{j+1} - x \\
		&= x^k_j -x  + \gamma {(b_{\tau(j)} - a^T_{\tau(j)} x^k_j ) \over \|a_{\tau(j)}\|_2^2} a_{\tau(j)} \\
		&= x^k_j -x + \gamma {( a^T_{\tau(j)}x + e_{\tau(j)} - a^T_{\tau(j)}x^k_j ) \over \|a_{\tau(j)}\|_2^2} a_{\tau(j)} \\
		&= \left(I - \gamma {a_{\tau(j)} a_{\tau(j)}^T \over \|a_{\tau(j)}\|_2^2} \right) (x^k_j -x) + \gamma e_{\tau(j)} {a_{\tau(j)} \over \|a_{\tau(j)}\|_2^2} \\
		&= \left(I - \gamma {a_{\tau(j)} a_{\tau(j)}^T \over \|a_{\tau(j)}\|_2^2} \right) \left(x^k_{j-1}-x + \gamma {( a^T_{\tau(j-1)}x + e_{\tau(j-1)} - a^T_{\tau(j-1)}x^k_{j-1} ) \over \|a_{\tau(j-1)}\|_2^2} a_{\tau(j-1)} \right) + \gamma e_{\tau(j)} {a_{\tau(j)} \over \|a_{\tau(j)}\|_2^2} \\
		&= \left(I - \gamma {a_{\tau(j)} a_{\tau(j)}^T \over \|a_{\tau(j)}\|_2^2} \right)  \left(I - \gamma {a_{\tau({j-1})} a_{\tau({j-1})}^T \over \|a_{\tau({j-1})}\|_2^2} \right) (x^k_{j-1} -x) \\
		& \qquad + \gamma e_{\tau(j-1)} \left(I - \gamma {a_{\tau(j)} a_{\tau(j)}^T \over \|a_{\tau(j)}\|_2^2} \right) {a_{\tau(j-1)} \over \|a_{\tau(j-1)}\|_2^2} + \gamma e_{\tau(j)} {a_{\tau(j)} \over \|a_{\tau(j)}\|_2^2}.
	\end{align*}
	Applying this relation repeatedly on $x^k_{m+1}$, we obtain the following equation for multi-step KZ iterations
	\begin{align*}
		x^k_{m+1} - x 
		&=  \left(I - \gamma {a_{\tau(m)} a_{\tau(m)}^T \over \|a_{\tau(m)}\|_2^2} \right) \left(I - \gamma {a_{\tau({m-1})} a_{\tau({m-1})}^T \over \|a_{\tau({m-1})}\|_2^2} \right) \cdots \left(I - \gamma {a_{\tau(1)} a_{\tau(1)}^T \over \|a_{\tau(1)}\|_2^2} \right) (x^k_1 -x) \\
		& \qquad + \gamma \sum_{i=0}^{m-1} e_{\tau(m-i)} \prod_{j=0}^{i-1}  \left(I - \gamma {a_{\tau({m-j})} a_{\tau({m-j})}^T \over \|a_{\tau({m-j})}\|_2^2} \right) {a_{\tau(m-i)} \over \|a_{\tau(m-i)}\|_2^2},
	\end{align*} with the convention $\prod_{p}^q = 1$ if $p > q$.
	\end{proof}
	
	\begin{algorithm}
		\caption{KZPT with period $m$}
		\begin{algorithmic}
			\Inputs $m\times N$ measurement matrix $A$, measurements $b=Ax +e$, sparsity level $s$, step size $\gamma$ for Kaczmarz inner iteration, step size $\lambda$ for IHT, the number of iterations $K$
			
			\Initialize $x^{1}_1 = 0$.
			 
		\ForOuterLoop
  
		\indent \indent	\ForLoop 
			\begin{align*}
				x^k_{j+1} := x^k_j + \gamma {(b_{\tau(j)} - a^T_{\tau(j)} x^k_j ) \over \|a_{\tau(j)}\|_2^2} a_{\tau(j)} \qquad \text{[Kaczmarz loop with step size $\gamma$]}	
			\end{align*}
			
		\indent \indent	\ForEnd
			\begin{align*}
				&g^k := x^k_{m+1} - x^k_1\\
				&u^k := x^k_1 + \lambda g^k \qquad \qquad \qquad \qquad \qquad \text{[IHT gradient step with step size $\lambda$]}
			\end{align*}

		\indent \indent	\textbf{Project $u^k$ to $s$-sparse vector sets}\\
			\[
			x^{k+1}_1 := T_s(u^k) \qquad \qquad \qquad \qquad \qquad \qquad \qquad \qquad \qquad \qquad \qquad \qquad 
			\]

            \ForEnd
            
			\Outputs The approximate $s$-sparse solution to $Ax=b$.
			
		\end{algorithmic}
		\label{alg:KZPT_period_m}
	\end{algorithm}

	\subsection{Warm-up: Randomly Subsampled Bounded Orthonormal Systems}

	\label{subsection:BOS}
	Let us consider a bounded orthonormal system (BOS), defined as any orthogonal matrix ${1 \over \sqrt{N}} \Phi \in \R^{N \times N}$ with $|\Phi_{ij}| \le 1$ for all entries $\Phi_{ij}$. This includes well-known systems such as the Hadamard transform. The randomly subsampled BOS sensing matrix $A \in \R^{m \times N}$ is constructed by selecting $m$ rows of $\Phi$ uniformly at random (or sampling rows at random without replacement). This type of sensing matrix $A$ typically satisfies the following \emph{restricted isometry property} (RIP), a key quantity in the convergence analysis of compressed sensing algorithms \cite{blumensath2009iterative,candes2006robust, needell2009cosamp, foucart2013invitation}. 
	
	\begin{dfn}
		An $m \times N$ matrix is said to satisfy the restricted isometry property with constant $\delta_s$ if $\delta_s \ge 0$ is the smallest constant such that 
		\[
		(1 - \delta_s)\|z\|_2^2 \le \|Az\|_2^2 \le (1 + \delta_s)\|z\|_2^2 
		\] for all $s$-sparse vectors $z \in \R^N$.
	\end{dfn}

	It turns out that the analysis for the randomly subsampled BOS case becomes particularly simple since the rows of the sampling matrix are orthogonal; the cross terms of the matrix product in \eqref{eq:main_equation_in_product_form} after expansion vanish if all $a_{\tau(j)}$ are orthogonal to each other. 
	This observation allows us to prove the following theorem for linear convergence of KZIHT and KZPT for the BOS case.
	
	\begin{thm}
		\label{thm:KZIHT_convergence}
		Let $A$ be an $m \times N$ randomly subsampled bounded orthonormal matrix. We run KZPT with measurements $b = Ax + e$ for some $s$-sparse vector $x \in \R^N$ and  noise vector $e$. Suppose the row selection rule $\tau$ is a permutation (either deterministic or possibly random) of the index set $\{1,2, \dots, m\}$. Let $\delta \in (0,1/2)$. 
		Then, for $\gamma = N/m$, $\lambda = 1$, and  period $m$, the KZPT iterates $x^k_1$ in $k$ satisfy
		\[
    	\|x^{k+1}_1 - x\| \le \left(2\delta\right)^k \|x^1_1 - x\| + {1 \over m} \cdot {2 \over 1 - 2\delta} \cdot \sup_{\Omega \subset [N], |\Omega| \le 2s} \left \|P_{\Omega} \left(   A^Te \right) \right \|_2,
    	\]
		 with probability at least $1 - {c \over N^3}$, as long as $m \ge C \delta^{-2} s \ln^4 N$ for some universal constants $c, C > 0$.

        In particular,  for $\gamma = N/m$, KZIHT  converges linearly in the number of epochs, $k$, with high probability.
	\end{thm}

	\begin{rem} [Optimal statistical bias]
	\label{rem:Optimal statistical bias for BOS}
	The residual error bound ${1 \over m} \cdot {2 \over 1 - 2\delta} \cdot \sup_{\Omega \subset [N], |\Omega| \le 2s} \left \|P_{\Omega} \left(   A^Te \right) \right \|_2$ in Theorem \ref{thm:KZIHT_convergence} for KZIHT with reshuffling  is of the same order of the optimal statistical bias for the sparse linear regression problem. The reformulation of Kaczmarz method as a stochastic gradient descent in \cite{needell2014stochastic} recasts the method as an SGD with the objective function $f(z) = {1 \over 2m} \|b - Az\|_2^2$. Note that $\nabla f(x) = {1 \over m} A^T(b-Ax) = {1 \over m} A^Te$. Thus, under Gaussian noise or more generally mean-zero i.i.d. sub-Gaussian assumption on the noise $e$, it is well-known that the residual error $O \left( \sup\limits_{\Omega \subset [N], |\Omega| \le 2s} \left \|P_{\Omega}  \nabla f(x) \right \|_2 \right)$ is of the order of $O \left(  \sqrt{s \log N \over m} \right)$, the optimal statistical bias for sparse linear regression \cite{loh2013regularized}.

	Mishchenko et al.  \cite{mishchenko2020random} also study the reshuffling Kaczmarz method but with soft-thresholding or, more generally, reshuffling proximal methods based on regularizers. Their theoretical guarantees, however, require regularizers to be strongly convex, which excludes an important case --- $\ell_1$ regularizer for soft-thresholding for sparse linear regression.
	
	\end{rem}

    \begin{rem}
         There are two types of randomly subsampled BOS in the literature, depending on how we sample the rows from the orthonormal matrix ${1 \over \sqrt{N}} \Phi \in \R^{N \times N}$. The rows can either be selected from ${1 \over \sqrt{N}} \Phi$ independently and uniformly at random (sampling with replacement), or a subset of $m$  rows of ${1 \over \sqrt{N}} \Phi$ can be selected independently and uniformly at random (sampling without replacement). Our model is the latter, and both the theoretical guarantees and practical performance of these two models are very similar \cite{rauhut2010compressive, foucart2013invitation}.  
    \end{rem}

	\begin{proof} [Proof of Theorem \ref{thm:KZIHT_convergence}]
	We first note that because the gradient step size $\lambda$ is set to $1$ and the period is $m$, KZPT is identical to KZIHT as one can see in Algorithm \ref{alg:KZIHT} and Algorithm \ref{alg:KZPT_period_m}, so we will present the proof for KZIHT.
	
	From the assumption that $\tau$ is a permutation, all $\tau(j)$ are distinct, so all the vectors $a_{\tau(j)}$ are orthogonal to each other. Then, in the expansion of the matrix product in \eqref{eq:main_equation_in_product_form} of Lemma \ref{lem:multi_step_equation}, the cross terms vanish.
	Hence, after we run $m$ number of $KZ$ iterations on $x^k_1$, the $k$-th iterate of KZIHT, we obtain
	\begin{align*}
		x^k_{m+1} - x 
		&= \left( I - \gamma \sum_{j=1}^m {a_{\tau({j})} a_{\tau({j})}^T \over \|a_{\tau({j})}\|_2^2} \right) (x^k_1 -x)  + \gamma \sum_{i=0}^{m-1} e_{\tau(m-i)} \prod_{j=0}^{i-1}  \left(I - \gamma {a_{\tau({m-j})} a_{\tau({m-j})}^T \over \|a_{\tau({m-j})}\|_2^2} \right) {a_{\tau(m-i)} \over \|a_{\tau(m-i)}\|_2^2}.
	\end{align*}
	Using the orthogonality of the vectors $a_{\tau(j)}$ again to the second term in the equation above gives us the following simplification: 
    \[
    \prod\limits_{j=0}^{i-1}  \left(I - \gamma {a_{\tau({m-j})} a_{\tau({m-j})}^T \over \|a_{\tau({m-j})}\|_2^2} \right) {a_{\tau(m-i)} \over \|a_{\tau(m-i)}\|_2^2} =  \left(I - \sum\limits_{j=0}^{i-1}  \gamma {a_{\tau({m-j})} a_{\tau({m-j})}^T \over \|a_{\tau({m-j})}\|_2^2} \right) {a_{\tau(m-i)} \over \|a_{\tau(m-i)}\|_2^2} = {a_{\tau(m-i)} \over \|a_{\tau(m-i)}\|_2^2}.
    \]
    Then, the second term reduces to
 $\gamma \sum\limits_{i=0}^{m-1} e_{\tau(m-i)} {a_{\tau(m-i)} \over \|a_{\tau(m-i)}\|_2^2}   = \gamma \sum\limits_{i=1}^m e_{i}  {a_{i} \over \|a_{i}\|_2^2}$. Thus, we have
	\begin{align}
		\label{eq:multi_level_eq}
		x^k_{m+1} - x 
		&= \left( I - \gamma \sum_{j=1}^m {a_j a_j^T \over \|a_j\|_2^2} \right) (x^k_1 -x)  + \gamma \sum\limits_{i=1}^m e_{i}  {a_{i} \over \|a_{i}\|_2^2}.
	\end{align}
    Also, since ${1 \over \sqrt{N}} \Phi$ is an orthogonal matrix, $\|a_j\|_2^2 = N$ for all $j$ which makes \eqref{eq:multi_level_eq} reduce to
	\[
	x^k_{m+1} - x = \left( I - \gamma \sum_{i=1}^m  {a_j a_j^T \over N
	} \right) (x^k_1 -x) + \gamma \sum_{i=1}^m e_{i}  {a_{i} \over N}.
	\]
    From the assumption on the parameter in the theorem, $\gamma = N/m$, we have
	\begin{align*}
		x^k_{m+1} - x 
		&= \left( I - {1 \over m} \sum_{j=1}^m a_j a_j^T \right) (x^k_1 -x) +  {1 \over m} \sum_{i=1}^m e_{i}  {a_{i}} \\
		&= \left( I - \left({1 \over \sqrt{m}} A\right)^T \left( {1 \over \sqrt{m}}A \right) \right) (x^k_1 -x) +  {1 \over m} A^Te.
	\end{align*} From this equation, we obtain the following multi-step iteration equation for $x^k_{m+1}$
	\begin{align*}
	    x^k_{m+1} 
	    &= x^k_1 + \left({1 \over \sqrt{m}} A\right)^T \left( {1 \over \sqrt{m}}A \right)(x - x^k_1) +  {1 \over m} A^Te.
	\end{align*}	
	After applying the threshold operator $T_s$ to $x^k_{m+1}$, we establish the iteration equation of KZIHT 
	\begin{align*}
	x^{k+1}_1 
	&= T_s \left( x^k_1 + \left({1 \over \sqrt{m}} A\right)^T \left( {1 \over \sqrt{m}}A \right)(x - x^k_1) +  {1 \over m} A^Te \right) \\
     &= T_s \left(  x^k_1 + {1 \over m}A^T(b - A x^k_1) \right),
	\end{align*}
	 which is identical to the iteration equation of IHT for noisy measurements. The rest of the proof for the linear convergence of KZPT with randomly sampled BOS hinges on a modification of a standard argument \cite{foucart2013invitation}. Because we will extend this argument substantially to prove the sub-Gaussian random matrix case later, we include it for the sake of completeness. 

	The well-established results in \cite{candes2006robust, rudelson2008sparse, dirksen2015} state that ${1 \over \sqrt{m}}A$ satisfies $3s$-RIP with $\delta_{3s} < 1/2$ with high probability as long as $m \ge C s\log^4 N$ for some universal constant $C > 0$. More precisely, we have the following lemma.
	
	\begin{lem} [Theorem 4.1 in \cite{dirksen2015}]
		\label{lem:BOS_lemma}
		Let $A$ be a $q \times N$ sampling matrix randomly subsampled BOS matrix ${1 \over \sqrt{N}} \Phi$ whose entries satisfying $|\Phi_{ij}| \le 1$. Then, for $\delta, \epsilon \in (0,1)$, if 
		\[
		q \ge C \delta^{-2} s \max \{ \ln^4 N, \ln (\epsilon^{-1}) \}
		\]  the restricted isometry constant $\delta_{s,q}$ of ${1 \over \sqrt{q}} A$ satisfies $\delta_{s,q} \le \delta$ with probability at least $1 - \epsilon$ for some universal constant $C > 0$.
	\end{lem}

    By choosing an appropriate universal constant $c > 0$ for $\epsilon = {c \over N^3}$, we have $\ln \epsilon^{-1} < \ln^4 N$. By Lemma \ref{lem:BOS_lemma}, this choice of $\epsilon$ and the assumption $m \ge C \delta^{-2} s \ln^4 N$ in the theorem make ${1 \over \sqrt{m}}A$ satisfy $3s$-RIP with a RIP constant $\delta_{3s} \le \delta$, with probability at least $1 - {c \over N^3}$. 
		
    Now, let $\Omega_{k+1}$ be the union of supports of $x^{k+1}_1$ and $x$, so $|\Omega_{k+1}| \le 2s$ since both $x^{k+1}_1$ and $x$ are $s$-sparse. For a subset $\Omega \subset [n]$, define $P_{\Omega}$ as the orthogonal projection operator onto the subspace of $\R^N$ spanned by the standard basis vectors restricted to the index set $\Omega$. Because $x^{k+1}_1$ is the closest $s$-sparse vector to $x^k_{m+1}$, 
	\[
		\|x^k_{m+1} - x^{k+1}_1\|_2^2 \le \|x^k_{m+1} - x\|_2^2.
	\]
	We expand the left hand side of above equation, $\|(x^k_{m+1} - x) - (x^{k+1}_1 - x) \|_2^2$ to obtain 
	\[
		\|x^{k+1}_1 - x\|_2^2 \le 2 \inner{x^k_{m+1} - x, x^{k+1}_1 - x}.
	\]
 The right hand side of this inequality is further bounded by the above as below.
	\begin{align*}
		&\inner{x^k_{m+1} - x, x^{k+1}_1 - x} \\
		&= \inner{ \left(I - \left({1 \over \sqrt{m}} A\right)^T \left( {1 \over \sqrt{m}}A \right) \right)(x^k_1 -x) +  {1 \over m} A^Te, x^{k+1}_1 - x} \\
		&= \inner{ \left(I - \left({1 \over \sqrt{m}} A\right)^T \left( {1 \over \sqrt{m}}A \right)  \right)(x^k_1 -x), x^{k+1}_1 - x} + \inner{  {1 \over m} A^Te, x^{k+1}_1 - x} \\
		&\le \delta_{3s} \|x^k_1 -x\|_2 \|x^{k+1}_1 - x\|_2  + \inner{  {1 \over m} A^Te, x^{k+1}_1 - x} \\
		&= \delta_{3s} \|x^k_1 -x\|_2 \|x^{k+1}_1 - x\|_2  + \inner{ {1 \over m} A^Te, P_{\Omega_{k+1}} \left( x^{k+1}_1 - x \right) } \\
		&= \delta_{3s} \|x^k_1 -x\|_2 \|x^{k+1}_1 - x\|_2 + \inner{ P_{\Omega_{k+1}} \left(  {1 \over m} A^Te \right), x^{k+1}_1 - x} \\
		&\le \delta_{3s} \|x^k_1 -x\|_2 \|x^{k+1}_1 - x\|_2 + \left \|P_{\Omega_{k+1}} \left( {1 \over m} A^Te \right) \right \|_2 \|x^{k+1}_1 - x\|_2\\
		&\le \delta_{3s} \|x^k_1 -x\|_2 \|x^{k+1}_1 - x\|_2 +  \sup_{\Omega \subset [N], |\Omega| \le 2s} \left \|P_{\Omega} \left( {1 \over m} A^Te \right) \right \|_2 \|x^{k+1}_1 - x\|_2.
	\end{align*} 
	 In the chain of inequalities above, the first inequality follows from the RIP of ${1 \over \sqrt{m}}A$ and Lemma 6.16 in \cite{foucart2013invitation}, which is a consequence of an equivalent formulation of RIP. 
If $\|x^{k+1}_1 - x\|_2 > 0$, we have
	\begin{align}
        \label{eq:contraction_bound_KZIHT}
	    \|x^{k+1}_1 - x\|_2 \le 2 \delta_{3s} \|x^k_1 -x\|_2 + 2 \sup_{\Omega \subset [N], |\Omega| \le 2s} \left \|P_{\Omega} \left(  {1 \over m} A^Te \right) \right \|_2.
	\end{align}
		
	Thus, mathematical induction on $k$ and the infinite geometric series formula yield
    \begin{align*}
	\|x^{k+1}_1 - x\| 
        &\le \left(2\delta_{3s}\right)^k \|x^1_1 - x\| + {2 \over 1 - 2\delta_{3s}} \cdot \sup_{\Omega \subset [N], |\Omega| \le 2s} \left \|P_{\Omega} \left( {1 \over m} A^Te \right) \right \|_2\\
        &\le \left(2\delta \right)^k \|x^1_1 - x\| + {2 \over 1 - 2\delta } \cdot \sup_{\Omega \subset [N], |\Omega| \le 2s} \left \|P_{\Omega} \left( {1 \over m} A^Te \right) \right \|_2, 
    \end{align*} where the second inequality is from $\delta_{3s} \le \delta$.
    Since $\delta< 1/2$ from the assumption in the theorem, this shows that KZIHT converges linearly in  the number of KZIHT iterations or epochs up to a small neighborhood, the size of which is proportional to the noise level for randomly subsampled bounded orthonormal systems. 
    \end{proof}

	\subsection{Sub-Gaussian Random Matrices}
	
	We have seen in Section \ref{subsection:BOS} that KZPT exhibits linear convergence when the measurement matrix is a randomly subsampled bounded orthonormal system by exploiting the orthogonality of its rows.
	
	It turns out that the idea of our proof extends to other popular sensing matrices such as Bernoulli (matrix whose entries are independent Rademacher random variables) or Gaussian random matrices because rows of these random matrices are almost orthogonal with high probability. 
	
	While this type of matrices may need more computational resources than BOS as they are unstructured, they come with stronger theoretical guarantees typically requiring only $C s\log (N/s)$ number of measurements for successful recovery, in contrast to the $C s \log^4 N$ requirement for the BOS case. In some applications, especially for large ambient dimensions $N$, the polylogarithmic factor difference could be of practical concern. 
	
	Motivated by the theoretical guarantees with fewer measurements, we present a linear convergence proof of KZPT when random matrices have fixed length, mean-zero, independent, and isotropic sub-Gaussian rows. These include important cases such as Bernoulli matrices or rescaled Gaussian matrices (a matrix whose rows are drawn uniformly at random from a sphere). If rows of the sensing matrix are not of the same $\ell_2$-norm, one may normalize and rescale rows and corresponding observations as needed. A similar, but more restricted class of random matrices has been used in \cite{haddock2022quantile} to model linear systems without sparsity constraints where the authors study KZ-based solvers. Sub-Gaussian random matrices are also widely used as a standard model in the theory of sparse linear regression \cite{loh2011high, loh2013regularized}, thereby providing another motivation for carrying out our analysis under this model.

	\begin{thm}
	\label{thm:main_theorem2}
    	Let $A$ be an $m \times N$ random matrices whose rows $a_i$ are mean-zero isotropic, independent sub-Gaussian with sub-Gaussian norm bounded by $K$, and $\|a_i\|_2 = \sqrt{N}$. We run KZPT with measurements $b = Ax + e$ for some $s$-sparse vector $x \in \R^N$ and the noise vector $e$. 
	    Suppose the row selection rule $\tau$ is a permutation (either deterministic or possibly random) of the index set $\{1,2, \dots, m\}$. Let  $\delta = \sqrt{3s \ln (N/s) (K\ln m)^2 \over m} < 1/4$.  Then with $\gamma =   {\delta \over 2 m (K \ln m)^2 N^{1/2}}$, $\lambda =  {2 N^{3/2} (K \ln m)^2 \over \delta}$ and period $m$, as long as 
        $m \ge  C_1 s \ln (N/s) (K\ln m)^2$, the iterates of KZPT satisfy
		\[
		\|x^{k+1}_1 - x\| \le \left(4\delta\right)^k \|x^1_1 - x\| + {2 \over 1 - 4\delta} \cdot \sup_{\Omega \subset [N], |\Omega| \le 2s} \left \|P_{\Omega} \left(  {1 \over m} A^Te \right) \right \|_2 + O \left(\sqrt{{s \ln (N/s) \over m }} \right)  {\|e\|_1 \over  (1 - 4\delta)m}
		\] with probability at least $1  - (m^2 + 2) \exp{\left(-C_2 \ln^2 m \right)}$ , where $C_1$ and $C_2$ are positive universal constants that only depend on $K$. 
	\end{thm}
	
	\begin{rem}
	     Theorem \ref{thm:main_theorem2} guarantees that KZPT converges linearly to the solution up to a small neighborhood of size
	    \begin{align}
	    \label{rem:residual_error_Bernoulli}
	         {1 \over 1 - 4\delta} \left(2\sup_{\Omega \subset [N], |\Omega| \le 2s} \left \|P_{\Omega} \left(  {1 \over m} A^Te \right) \right \|_2 + O \left(\sqrt{{s \ln (N/s) \over m }} \right) {\|e\|_1 \over m} \right). 
	    \end{align}
	   
	    Consider common noise models such as Gaussian noise or, more generally, noise vectors whose components are mean-zero i.i.d. sub-Gaussian with sub-Gaussian norm $K_e$. From the Hoeffding's inequality (see for example, Theorem 2.6.2 in \cite{vershynin2018high}), we have ${\|e\|_1 \over m} = {1 \over m} \sum\limits_i^m |e_i| \lesssim \mathbb{E}|e_1| + {CK_e \over \sqrt{m}}$ with high probability, where $C$ is a universal constant. 
	    Then, as with the randomly subsampled BOS case in Remark \ref{rem:Optimal statistical bias for BOS}, since $K_e$ is typically small for reasonable noise models,
	    the residual error \eqref{rem:residual_error_Bernoulli} is of the order $O \left(\sqrt{{s \log N \over m }} \right)$, which is again of the same order of the optimal statistical bias. According to our knowledge, this is the first work  to propose a thresholding-based reshuffling SGD for sparse linear regression that comes with the optimal statistical bias. 
	    
	\end{rem}

	\begin{proof} [Proof of Theorem \ref{thm:main_theorem2}]
			Since the period is $m$, we begin with the KZ iterations in Algorithm \ref{alg:KZPT_period_m}. Suppose we start from $x^k_1$. After we run $m$ number of Kaczmarz iterations, we obtain equation \eqref{eq:main_equation_in_product_form}, which is written below again for convenience.
		\begin{align*}
			 x^k_{m+1} - x 
		&=  \left(I - \gamma {a_{\tau(m)} a_{\tau(m)}^T \over \|a_{\tau(m)}\|_2^2} \right) \left(I - \gamma {a_{\tau({m-1})} a_{\tau({m-1})}^T \over \|a_{\tau({m-1})}\|_2^2} \right) \cdots \left(I - \gamma {a_{\tau(1)} a_{\tau(1)}^T \over \|a_{\tau(1)}\|_2^2} \right) (x^k_1 -x)\\
		\nonumber
		& \qquad + \gamma \sum_{i=0}^{m-1} e_{\tau(m-i)} \prod_{j=0}^{i-1}  \left(I - \gamma {a_{\tau({m-j})} a_{\tau({m-j})}^T \over \|a_{\tau({m-j})}\|_2^2} \right) {a_{\tau(m-i)} \over \|a_{\tau(m-i)}\|_2^2}.
		\end{align*} 

		Define $A_i := a_i a_i^T$ for $1 \le i \le m$. Because $\|a_i\|_2^2 = N$, the above equation can be rewritten as
		\begin{align}
			\nonumber
			&x^k_{m+1} - x \\
			\nonumber
			&=  \left(I - \gamma {A_{\tau(m)} \over N} \right) \left(I - \gamma {A_{\tau(m-1)} \over N} \right) \cdots \left(I - \gamma {A_{\tau(1)} \over N} \right) (x^k_1 -x) \\
                &\qquad + \gamma \sum_{i=0}^{m-1} e_{\tau(m-i)} \prod_{j=0}^{i-1}  \left(I - \gamma {A_{\tau({m-j})}  \over N} \right) {a_{\tau(m-i)} \over N} \\
			\label{eq:expanded_multi_step_eq1}
			&= \left(I -  {\gamma \over N} \sum_{j=1}^m A_j +   {\gamma^2 \over N^2} \sum_{ 1 \le i_1 < i_2 \le m}^m A_{\tau(i_2)} A_{\tau(i_1)}  -  {\gamma^3 \over N^3} \sum_{ 1 \le i_1 < i_2 < i_3 \le m}^m A_{\tau(i_3)} A_{\tau(i_2)} A_{\tau(i_1)} + \cdots \right) (x^k_1 -x) \\
			\label{eq:expanded_multi_step_eq2}
			& \qquad + \gamma \sum_{i=0}^{m-1} e_{\tau(m-i)} \prod_{j=0}^{i-1}  \left(I - \gamma {A_{\tau({m-j})}  \over N} \right) {a_{\tau(m-i)} \over N}.
		\end{align} 
		
		We want to have an upper bound on the operator norm of these cross terms in the matrix products in \eqref{eq:expanded_multi_step_eq1}. For notational convenience, we denote the cross terms by $B_\tau$, i.e., 
	\begin{align}
             \label{eq:cross_terms_in_product1}
		B_\tau :=  {\gamma^2 \over N^2} \sum_{ 1 \le i_1 < i_2 \le m}^m A_{\tau(i_2)} A_{\tau(i_1)}  -  {\gamma^3 \over N^3} \sum_{ 1 \le i_1 < i_2 < i_3 \le m}^m A_{\tau(i_3)} A_{\tau(i_2)} A_{\tau(i_1)} + \cdots.
	\end{align}

		First, we prove a useful lemma generalizing the result by Milman and  Schechtman \cite{milman2009asymptotic} about the almost orthogonal property of random vectors drawn uniformly from a sphere.
			
\begin{lem}
	\label{lem:cross_term_bound}
	Let $v$ and $w$ be $N$-dimensional independent, mean-zero, sub-Gaussian random vectors with sub-Gaussian norm bounded by $K$. Then, we have
	\[
	\mathbb{P} \left( | v^T w | >  K^2 \sqrt{N} \ln^2 m \right)  \le  \exp{\left(-C \ln^2 m \right)},
	\] for some universal constant $C > 0$. 
\end{lem}

\begin{proof} [Proof of Lemma \ref{lem:cross_term_bound}]
	From the Gaussian Chaos comparison lemma in Vershynin \cite{vershynin2018high} [Lemma 6.2.3 in Vershynin], we have
	\[
		\mathbb{E} \exp \left( \lambda v^T w \right) \le \mathbb{E} \exp (CK^2 \lambda g^Tg'),
    \] for any $\lambda \in \R$ where $g$ and $g'$ are $N$-dimensional standard normal Gaussian random vectors and $C > 0$ is a universal constant. On the other hand, applying Lemma 6.2.2 in Vershynin \cite{vershynin2018high} to above bound gives 
    \[
  	  \mathbb{E} \exp (CK^2 \lambda g^Tg') \le \exp(C' K^4 \lambda^2 \|I_{N \times N}\|_F^2) =  \exp(C' K^4 N \lambda^2 )
    \] for any $K^2\lambda < c/\|I_{N \times N}\| = c$ for some possibly other universal constants $c, C' > 0$. Hence, we have 
    \begin{align}
    	\label{ineq:MGF-Gaussian Chaos bound}
    	\mathbb{E} \exp \left( \lambda v^T w \right) \le  \exp(C' K^4 N \lambda^2 ),
    \end{align} for any $\lambda < c/K^2$.

	By the Chernoff's inequality and from the bound \eqref{ineq:MGF-Gaussian Chaos bound}, we have
	\[
		\mathbb{P} (v^T w \ge t/2) \le \exp(-\lambda t/2) \mathbb{E} (\lambda v^T w) \le  \exp(-\lambda t/2) \exp(C' K^4 N \lambda^2 ),
	\] for any $\lambda < c/K^2$.
	After we optimize over $\lambda$ in the term in the right hand side of the last inequality, we obtain
	\[
		\mathbb{P} (v^T w \ge t/2) \le \exp { \left[ -c' \min \left \{ {t^2 \over K^4 N}, {t \over K^2} \right\} \right]},
	\] for some possibly another universal constant $c' > 0$. By applying the same argument to $-v$ and $w$ and combining above bound, we have 
	\[
	\mathbb{P} (|v^T w| \ge t/2) \le 2\exp { \left[ -c' \min \left \{ {t^2 \over K^4 N}, {t \over K^2} \right\} \right]}.
	\]
    By taking $t = 2K^2 \sqrt{N} \ln^2 m $,  $|v^T w|$ can be bounded as follows. 
	\[
	\mathbb{P} \left( |v^T w| > K^2 \sqrt{N} \ln^2 m \right) \le \exp{\left(-C'' \ln^2 m \right)}
	\] for some universal constant $C'' > 0$.
\end{proof}

Next, we present how to bound $B_\tau$ using Lemma \ref{lem:cross_term_bound}.
		\begin{lem}
		\label{lem:Bound_of_cross_terms}
			If $\gamma \le {1 \over 2m} \left({N \over K^4 \ln^4 m}\right)^{1/2}$, then there exist a universal constant $C > 0$ such that 
			\[
			\|B_\tau\|  <    2 \gamma^2 m^2 \left({K^4 \ln^4 m \over N}\right)^{1/2}
			\] uniformly for all permutation $\tau$ with probability at least $1 -  m^2  \exp{\left(-C \ln^2 m \right)}$.
		\end{lem}

		\begin{proof}[proof of Lemma \ref{lem:Bound_of_cross_terms}] 
			First, note that $\|A_{i_1} A_{i_2}\| = \|a_{i_1} a_{i_1}^T a_{i_2} a_{i_2}^T\| \le  |a_{i_1}^T a_{i_2}| \|a_{i_1} a_{i_2}^T\|$. From Lemma \ref{lem:cross_term_bound}, $|a_{i_1}^T a_{i_2}| \le K^2 \sqrt{N} \ln^2 m$ with probability at least $1 -  \exp{\left(-C \ln^2 m \right)}$. Since there are $m \choose 2$ number of such pairs $(i_1, i_2)$, taking a union bound gives
			\[
			\mathbb{P} \left(  \text{ For all $i_1 \neq i_2$, } |a_{i_1}^T a_{i_2}| \le K^2 \sqrt{N} \ln^2 m \right) \ge 1 -  m^2  \exp{\left(-C \ln^2 m \right)}.
			\] 
			Also, it is easy to see that $\|a_{i_1} a_{i_2}^T\| \le N$ because $\|a_i\|_2 = \sqrt{N}$ for all $i$. 
			
			Combining these bounds, we have $\|A_{i_1} A_{i_2}\| \le K^2 \sqrt{N} \ln^2 m  \cdot N \le K^2 \ln^2 m \cdot N^{3/2}$ with probability exceeding $1 -  m^2  \exp{\left(-C \ln^2 m \right)}$. 
			Similarly, it is easy to check that
			$\|A_{i_1} A_{i_2} \dots A_{i_p} \| \le (K^2 \ln^2 m)^{(p-1)} N^{(p-1)/2} \cdot N \le (K \ln m)^{2(p-1)} \cdot N^{(p+1)/2}$.
			
			Then, the cross terms can be controlled using these bounds: 
			\[
			\left \|{\gamma^2 \over N^2} \sum_{ 1 \le i_1 < i_2 \le m}^m A_{\tau(i_2)} A_{\tau(i_1)} \right \| \le {\gamma^2 \over N^2} \sum_{ 1 \le i_1 < i_2 \le m}^m \| A_{\tau(i_2)} A_{\tau(i_1)} \| \le  {\gamma^2 \over N^2 } \cdot {m^2 \over 2! }  \cdot K^2 \ln^2 m \cdot N^{3/2}.
			\]
			In general, each term in \eqref{eq:cross_terms_in_product1} can be bounded as follows:
			\begin{align*}
				\left \|{\gamma^p \over N^p} \sum_{ 1 \le i_1 < i_2 < \cdots < i_p \le m}^m A_{\tau(i_p)} A_{\tau(i_{p-1})} \cdots A_{\tau(i_1)} \right \| &\le {\gamma^p \over N^p} \sum_{ 1 \le i_1 < i_2 < \cdots < i_p \le m}^m \| A_{\tau(i_p)} A_{\tau(i_{p-1})} \cdots A_{\tau(i_1)} \| \\
				& \le  {\gamma^p m^p \over N^p p!}  (K \ln m)^{2(p-1)} \cdot N^{(p+1)/2}\\
				& \le (\gamma m)^p \left({K^4 \ln^4 m \over N} \right)^{(p-1)/2}.
			\end{align*}
			Note that the above bounds hold, with probability at least $1 -  m^2  \exp{\left(-C \ln^2 m \right)}$, for all permutations simultaneously.
			
			The geometric sum formula combined with previous bounds gives the following upper bound for $\|B_\tau\|$, 
			\begin{equation}
				\label{eq:Bound_for_B}
				\| B_\tau \| < { \gamma^2 m^2 ({K^4 \ln^4 m \over N})^{1/2} \over 1 - \gamma m \left({K^4 \ln^4 m \over N}\right)^{1/2} }  \le  2 \gamma^2 m^2 \left({K^4 \ln^4 m \over N}\right)^{1/2}
			\end{equation}
			if $\gamma \le {1 \over 2m} \left({N \over K^4 \ln^4 m}\right)^{1/2}$, for all permutation $\tau$. This proves the lemma.
		\end{proof}
		
	   As we denoted by $B_\tau$ the cross terms in the matrix products in \eqref{eq:expanded_multi_step_eq1}, we denote the cross terms in the matrix products $\prod\limits_{j=0}^{i-1}  \left(I - \gamma {A_{\tau({m-j})}  \over N} \right)$ in  \eqref{eq:expanded_multi_step_eq2} by, 
		\[
		B_{\tau(i)} :=  {\gamma^2 \over N^2} \sum_{m-i+1 \le i_1 < i_2 \le m}^{m} A_{\tau(i_2)} A_{\tau(i_1)}  -  {\gamma^3 \over N^3} \sum_{m-i+1 \le i_1 < i_2 < i_3 \le m}^{m} A_{\tau(i_3)} A_{\tau(i_2)}  A_{\tau(i_1)} + \cdots.
		\] 
		\begin{lem}
		\label{lem:Bound_of_cross_terms2}
			If $\gamma \le {1 \over 2m} \left({N \over K^4 \ln^4 m}\right)^{1/2}$, then there exist a universal constant $C > 0$ such that 
			\[
			\|B_{\tau(i)} \|  <  2 \gamma^2 m^2 \left({K^4 \ln^4 m \over N}\right)^{1/2},
			\] uniformly for all permutation $\tau$ and all $i \in [m]$ with probability at least $1 -  m^2  \exp{\left(-C \ln^2 m \right)}$.
		\end{lem}
		The proof of this lemma follows exactly the same argument used for the proof of Lemma \ref{lem:Bound_of_cross_terms}, so it is omitted.
		
		Now continuing from \eqref{eq:expanded_multi_step_eq1} and using the convention $\sum\limits_{i=p}^q = 0$ if $p > q$, we have
	\begin{align*}
		&x^k_{m+1} - x \\
            &= \left(I -  {\gamma \over N} \sum_{j=1}^m A_j + B_\tau \right) (x^k_1 -x) + \gamma \sum_{i=0}^{m-1} e_{\tau(m-i)}   \left(I - {\gamma \over N} \sum_{j=m-i+1}^{m} A_{\tau(j)} + B_{\tau(i)} \right) {a_{\tau(m-i)} \over N}\\
			&= x^k_1 - x  -{\gamma \over N}          \sum_{j=1}^m A_j (x^k_1 -x) + B_\tau (x^k_1 -x) + \gamma \sum_{i=0}^{m-1} e_{\tau(m-i)}   \left(I - {\gamma \over N} \sum_{j=m-i+1}^{m} A_{\tau(j)} + B_{\tau(i)} \right) {a_{\tau(m-i)} \over N},
		\end{align*}
		which implies 
		\[
		g^k = x^k_{m+1} - x^k_1 =  -{\gamma \over N} \sum_{j=1}^m A_j (x^k_1 -x) + B_\tau (x^k_1 -x) + \gamma \sum_{i=0}^{m-1} e_{\tau(m-i)}   \left(I - {\gamma \over N} \sum_{j=m-i+1}^{m} A_{\tau(j)} + B_{\tau(i)} \right) {a_{\tau(m-i)} \over N}.
		\]
	This leads to
	\begin{align*}
	    u^k &= x^k_1 + \lambda g^k \\
            &= x^k_1  -{\lambda \gamma \over N} \sum_{j=1}^m A_j (x^k_1 -x) + \lambda B_\tau (x^k_1 -x)+ \lambda\gamma \sum_{i=0}^{m-1} e_{\tau(m-i)}   \left(I - {\gamma \over N} \sum_{j=m-i+1}^{m} A_{\tau(j)} + B_{\tau(i)} \right) {a_{\tau(m-i)} \over N},
	\end{align*}	
	so we obtain
		\begin{align*}
		    u^k - x &= \left( I -{\lambda \gamma \over N} \sum_{j=1}^m A_j + \lambda B_\tau \right) (x^k_1 -x) + \lambda\gamma \sum_{i=0}^{m-1} e_{\tau(m-i)}   \left(I - {\gamma \over N} \sum_{j=m-i+1}^{m} A_{\tau(j)} + B_{\tau(i)} \right) {a_{\tau(m-i)} \over N}\\
		    &=  \left( I -{\lambda \gamma \over N} A^TA + \lambda B_\tau \right) (x^k_1 -x) + {\lambda\gamma \over N} \sum_{i=0}^{m-1} e_{\tau(m-i)}   \left(I - {\gamma \over N} \sum_{j=m-i+1}^{m} A_{\tau(j)} + B_{\tau(i)} \right) {a_{\tau(m-i)}}.
		\end{align*}

        To proceed further, we use the following lemma from \cite{foucart2013invitation}. 
        \begin{lem} [Theorem 9.6 in \cite{foucart2013invitation}]
		\label{lem:SubGaussian_lemma}
		Let $A$ be a $q \times N$ sampling matrix whose rows are independent, mean-zero isotropic sub-Gaussian with sub-Gaussian norm bounded by $K$.
		
	Then, for given $\epsilon, \delta \in (0,1)$, if
		\[
		q \ge C \delta^{-2} (s \ln (N/s) + \ln (2 \epsilon^{-1})),
		\]  the restricted isometry constant $\delta_{s,q}$ of ${1 \over \sqrt{q}} A$ satisfies $\delta_{s,q} \le \delta$ with probability at least $1 - \epsilon$ for some constant $C>0$ that only depends on $K$. 
	\end{lem}

     We set $\epsilon = 2 \exp (-\delta^2m /2C)$.  Then, by Lemma \ref{lem:SubGaussian_lemma}, ${1 \over \sqrt{m}}$ A satisfies the RIP with $\delta_{3s} \le \delta$ with probability at least $1 - \epsilon$ if $m \ge 6C \delta^{-2} s \ln (N/s)$.

    From the assumption in the theorem, 
    $m \ge  C_1 s \ln (N/s) (K\ln m)^2 \ge 48 s \ln (N/s) (K\ln m)^2$ by taking $C_1$ sufficiently large. 
    This requirement on $m$ makes possible to set $\delta$ as in the theorem, $\delta =  \sqrt{3s \ln (N/s) (K\ln m)^2 \over m} < 1/4$ which is equivalent to $m \ge 48 s \ln (N/s) (K\ln m)^2$.
    
    Moreover, it is easy to check that $m \ge  C_1 s \ln (N/s) (K\ln m)^2$ implies $m \ge 6C \delta^{-2} s \ln (N/s)$ for sufficiently large $C_1$. Hence, ${1 \over \sqrt{m}}$ A satisfies the RIP with $\delta_{3s} \le \delta$ with probability at least 
    \[
        1 - \epsilon =  1 - 2 \exp \left( - 3s \ln (N/s) (K\ln m)^2 /2C \right) \ge 1 - 2 \exp \left(-C'\ln m^2 \right),
    \]
     where $C' > 0$ is a constant that only depends on $K$. 

    Now, since ${1 \over \sqrt{m}} A$ satisfies the RIP with $\delta_{3s} \le \delta$ with high probability, we want the step sizes $\gamma, \lambda$ to satisfy ${\lambda \gamma \over N} = {1 \over m}$ and $\lambda \|B_\tau\| \le \delta$. These requirements and the upper bound for $B_\tau$ in Lemma \ref{lem:Bound_of_cross_terms} imply that the step size $\gamma$ for the Kaczmarz loop should be chosen so that $\gamma \le  {\delta \over 2m (K \ln m)^2 N^{1/2}}$. For simplicity, set $\gamma =  {\delta \over 2 m (K \ln m)^2 N^{1/2}}$. This in turn gives the step size $\lambda$ in the IHT gradient step, $\lambda = {2 N^{3/2} (K \ln m)^2 \over \delta}$. Intuitively, this choice of step sizes makes sense because the ${\lambda \gamma \over N} A^TA$ term depends on $\gamma$  linearly, whereas $B_\tau$ depends on the small step size $\gamma$ quadratically, which makes the term ${\lambda \gamma \over N} A^TA$ dominate.
		
	Hence, we  establish
		\begin{align*}
		    &u^k - x\\
		    &=  \left( I -{1 \over m} A^TA + \lambda B_\tau \right) (x^k_1 -x) + {1 \over m} \sum_{i=0}^{m-1} e_{\tau(m-i)}   \left(I - {\gamma \over N} \sum_{j=m-i+1}^{m} A_{\tau(j)} + B_{\tau(i)} \right) {a_{\tau(m-i)}} \\
		    &= \left(I - \left({1 \over \sqrt{m}} A\right)^T \left( {1 \over \sqrt{m}}A \right) + \lambda B_\tau \right)(x^k_1 -x) + {1 \over m} \sum_{i=0}^{m-1} e_{\tau(m-i)}   \left(I - {\gamma \over N} \sum_{j=m-i+1}^{m} A_{\tau(j)} + B_{\tau(i)} \right) {a_{\tau(m-i)}},
		\end{align*}
	which is equivalent to the gradient step equation for IHT except for an additional term $\lambda B (x^k_1 -x)$ and an error term originated from the noise $e$ in the measurements. 
  
	The rest of the proof for the linear convergence of KZPT with sub-Gaussian sensing matrices relies on the argument used in the proof of Theorem \ref{thm:KZIHT_convergence} combined with Lemma \ref{lem:Bound_of_cross_terms} and \ref{lem:Bound_of_cross_terms2} as follows.
		
	Because $x^{k+1}_1$ is the closest $s$-sparse vector to $u^k$, 
		\[
		\|u^k - x^{k+1}_1\|_2^2 \le \|u^k - x\|_2^2.
		\]
	We expand the left hand side of above equation, $\|(u^k - x) - (x^{k+1}_1 - x) \|_2^2$ to obtain 
		\[
		\|x^{k+1}_1 - x\|_2^2 \le 2 \inner{u^k - x, x^{k+1}_1 - x}.
		\]
	Let $\Omega_{k+1}$ be the union of supports of $x^{k+1}_1$ and $x$ and $P_{\Omega_{k+1}}$ be the corresponding orthogonal projection. The right hand side of this inequality is further bounded by the above as below.
		\small
		\begin{align*}
			&\inner{u^k - x, x^{k+1}_1 - x} \\
			&= \inner{ \left(I - \left({1 \over \sqrt{m}} A\right)^T \left( {1 \over \sqrt{m}}A \right)  \right)(x^k_1 -x), x^{k+1}_1 - x} + \inner{  \lambda B_\tau (x^k_1 -x), x^{k+1}_1 - x} \\
			& \qquad +   \inner{ {1 \over m} \sum_{i=0}^{m-1} e_{\tau(m-i)}   \left(I - {\gamma \over N} \sum_{j=m-i+1}^{m} A_{\tau(j)} + B_{\tau(i)} \right) {a_{\tau(m-i)}}, x^{k+1}_1 - x}  \\
			&\stackrel{(i)}{\le} \delta_{3s} \|x^k_1 -x\|_2 \|x^{k+1}_1 - x\|_2 + \inner{  \lambda B_\tau (x^k_1 -x), x^{k+1}_1 - x} \\
            &\qquad +  \inner{ {1 \over m} \sum_{i=0}^{m-1} e_{\tau(m-i)}   \left(I - {\gamma \over N} \sum_{j=m-i+1}^{m} A_{\tau(j)} + B_{\tau(i)} \right) {a_{\tau(m-i)}}, x^{k+1}_1 - x}\\
            &\le \delta \|x^k_1 -x\|_2 \|x^{k+1}_1 - x\|_2 + \inner{  \lambda B_\tau (x^k_1 -x), x^{k+1}_1 - x} \\
            &\qquad +  \inner{ {1 \over m} \sum_{i=0}^{m-1} e_{\tau(m-i)}   \left(I - {\gamma \over N} \sum_{j=m-i+1}^{m} A_{\tau(j)} + B_{\tau(i)} \right) {a_{\tau(m-i)}}, x^{k+1}_1 - x}\\
			&\le \delta \|x^k_1 -x\|_2 \|x^{k+1}_1 - x\|_2 + \lambda \|B_\tau (x^k_1 -x)\|_2 \|x^{k+1}_1 - x\|_2 \\
			&\qquad + \inner{{1 \over m} \sum_{i=0}^{m-1} e_{\tau(m-i)}   \left(I - {\gamma \over N} \sum_{j=m-i+1}^{m} A_{\tau(j)} + B_{\tau(i)} \right) {a_{\tau(m-i)}}, x^{k+1}_1 - x}\\
			&= \delta \|x^k_1 -x\|_2 \|x^{k+1}_1 - x\|_2 + \lambda \|B_\tau\| \|x^k_1 -x\|_2 \|x^{k+1}_1 - x\|_2 \\
			&\qquad +  \inner{{1 \over m} \sum_{i=0}^{m-1} e_{\tau(m-i)}   a_{\tau(m-i)}, P_{\Omega_{k+1}}\left( x^{k+1}_1 - x\right) }- {1 \over m} \sum_{i=0}^{m-1} e_{\tau(m-i)}  {\gamma \over N} \sum_{j=m-i+1}^{m} \inner{   A_{\tau({j})}   a_{\tau(m-i)}, x^{k+1}_1 - x} \\
			&\qquad +  {1 \over m} \sum_{i=0}^{m-1}  e_{\tau(m-i)}  \inner{B_{\tau(i)}  a_{\tau(m-i)}, x^{k+1}_1 - x}\\
			&\le \delta \|x^k_1 -x\|_2 \|x^{k+1}_1 - x\|_2 + \lambda \|B_\tau\| \|x^k_1 -x\|_2 \|x^{k+1}_1 - x\|_2 \\
			&\quad +  \left \|P_{\Omega_{k+1}}\left( {1 \over m} \sum_{i=0}^{m-1} e_{\tau(m-i)}   a_{\tau(m-i)} \right)\right\| \left\|  x^{k+1}_1 - x \right\|_2 + {1 \over m} \sum_{i=0}^{m-1} |e_{\tau(m-i)}|  {\gamma \over N} \sum_{j=m-i+1}^{m} \| A_{\tau({j})}  \| \|a_{\tau(m-i)}\|_2 \| x^{k+1}_1 - x \|_2 \\
			&\quad+  {1 \over m} \sum_{i=0}^{m-1} |e_{\tau(m-i)}|  \| B_{\tau(i)} \| \|a_{\tau(m-i)}\|_2 \| x^{k+1}_1 - x \|_2\\
			&\stackrel{(ii)}{<} \delta \|x^k_1 -x\|_2 \|x^{k+1}_1 - x\|_2 + \lambda \|B_\tau\| \|x^k_1 -x\|_2 \|x^{k+1}_1 - x\|_2 \\
			&\quad +  \left \|P_{\Omega_{k+1}}\left({1 \over m} \sum_{i=0}^{m-1} e_{\tau(m-i)}   a_{\tau(m-i)} \right)\right\| \left\|  x^{k+1}_1 - x \right\|_2 \\
                &\quad + {1 \over m} \sum_{i=0}^{m-1} |e_{\tau(m-i)}|  {\delta \over 2 m (K \ln m)^2 N^{3/2}} \sum_{j=m-i+1}^{m} N \cdot N^{1/2} \cdot \| x^{k+1}_1 - x \|_2 \\
			&\quad +  {1 \over m} \sum_{i=0}^{m-1} |e_{\tau(m-i)}| {\delta^2 \over 2 N^{3/2} (K \ln m)^2 } \cdot N^{1/2} \cdot \| x^{k+1}_1 - x \|_2\\
			&\le \delta \|x^k_1 -x\|_2 \|x^{k+1}_1 - x\|_2 + \lambda \|B_\tau\| \|x^k_1 -x\|_2 \|x^{k+1}_1 - x\|_2 \\
			&\quad +  \left \|P_{\Omega_{k+1}}\left({1 \over m} \sum_{i=1}^m e_i a_i \right)\right\| \left\|  x^{k+1}_1 - x \right\|_2 + {1 \over m} \sum_{i=1}^m |e_i|  {\delta \over 2 (K \ln m)^2 }  \| x^{k+1}_1 - x \|_2 \\
			&\quad +  {1 \over m} \sum_{i=1}^m |e_i|  {\delta^2 \over 2 N (K \ln m)^2 }  \cdot \| x^{k+1}_1 - x \|_2\\
			&\le \delta \|x^k_1 -x\|_2 \|x^{k+1}_1 - x\|_2 + \lambda \|B_\tau\| \|x^k_1 -x\|_2 \|x^{k+1}_1 - x\|_2 \\
			&\quad + \left \|P_{\Omega_{k+1}}\left({1 \over m} \sum_{i=1}^m e_i a_i \right)\right\| \left\|  x^{k+1}_1 - x \right\|_2 + \left({\delta \over 2  (K \ln m)^2 } +  {\delta^2 \over 2N (K \ln m)^2 } \right) {\|e\|_1 \over m} \cdot  \| x^{k+1}_1 - x \|_2.
		\end{align*} 
         \normalsize
         
  In the chain of inequalities above, $(i)$ follows from Lemma 6.16 in \cite{foucart2013invitation} and $(ii)$ follows from the choice of $\gamma = {\delta \over 2 m (K \ln m)^2 N^{1/2}}$ and Lemma \ref{lem:Bound_of_cross_terms2}. 
		
	Thus, if $\|x^{k+1}_1 - x\|_2 > 0$, we have
		\begin{align}
    		\nonumber
		    &\|x^{k+1}_1 - x\|_2 \\
		    \nonumber
		    &\le 2(\delta \|x^k_1 -x\|_2 + \lambda \|B_\tau\| \|x^k_1 -x\|_2) \\
		    \nonumber
		    &\qquad  + 2\sup_{\Omega \subset [N], |\Omega| \le 2s} \left \|P_{\Omega} \left(  {1 \over m} \sum_{i=1}^m e_i a_i \right) \right \|_2 + \left({\delta \over   (K \ln m)^2 } +  {\delta^2 \over N (K \ln m)^2} \right)  {\|e\|_1 \over m}\\
		    \label{eq:KZIHT_Bernoulli_three_terms}
		    &\le 4\delta\|x^k_1 -x\|_2 + 2\sup_{\Omega \subset [N], |\Omega| \le 2s} \left \|P_{\Omega} \left(  {1 \over m} A^Te \right) \right \|_2 + \left({\delta \over   (K \ln m)^2 } +  {\delta^2 \over N (K \ln m)^2 } \right)  {\|e\|_1 \over m},
		\end{align}
		where we have used $\lambda \|B_\tau\| < \delta$ due to the choice of $\lambda$ and Lemma \ref{lem:Bound_of_cross_terms}. 
		
	 Using the choice of $\delta$ to the third term in \eqref{eq:KZIHT_Bernoulli_three_terms}, we establish
	\begin{align*}
	    &\|x^{k+1}_1 - x\|_2 \\
	    &\le  4\delta\|x^k_1 -x\|_2 + 2 \sup_{\Omega \subset [N], |\Omega| \le 2s} \left \|P_{\Omega} \left(  {1 \over m} A^Te \right) \right \|_2 \\
	    &\qquad + 2 \left({1 \over K\ln m} \sqrt{{3s \ln (N/s) \over m  }} +  {3s \ln (N/s) \over 2mN  } \right)  {\|e\|_1 \over m}\\
	    &\le  4\delta\|x^k_1 -x\|_2 + 2 \sup_{\Omega \subset [N], |\Omega| \le 2s} \left \|P_{\Omega} \left(  {1 \over m} A^Te \right) \right \|_2 + O \left(\sqrt{{s \ln (N/s) \over m }} \right)  {\|e\|_1 \over m}. 
	\end{align*}

    Applying mathematical induction on $k$ combined with the infinite geometric series formula gives
	\[
	\|x^{k+1}_1 - x\| \le \left(4\delta\right)^k \|x^1_1 - x\| + {2 \over 1 - 4\delta} \sup_{\Omega \subset [N], |\Omega| \le 2s} \left \|P_{\Omega} \left(  {1 \over m} A^Te \right) \right \|_2 + O \left(\sqrt{{s \ln (N/s) \over m }} \right)  {\|e\|_1 \over  (1 - 4\delta)m}
	\]

	 Since $\delta < 1/4$, this shows the linear convergence of KZPT for random matrices whose rows are of fixed length, mean-zero, independent, isotropic sub-Gaussian, with probability exceeding 
    $1 - m^2  \exp{\left(-C \ln^2 m \right)} - 2 \exp \left(-C'\ln m^2 \right) \ge 1 - (m^2 + 2)\exp{\left(-C_2 \ln^2 m \right)}$ for some constant $C_2 > 0$. 
    \end{proof}

	\section{KZPT with Optimal Period}
	In this section, we investigate the benefits of using nontrivial periods $p$ for KZPT, especially in the low noise regime. We recall that Algorithm \ref{alg:KZPT} details KZPT.

	The convergence analysis of KZPT with nontrivial period $p$ is similar to the $p = m$ case discussed in Section \ref{section:KZPT_period_1}. We construct the sampling matrix $A$ by selecting a subset of $m$  rows of the BOS matrix uniformly at random, as in Section \ref{subsection:BOS}. For simplicity, we assume that $m$ is divisible by the period $p$. We denote $A^{(i)}$ the $p \times N$ submatrix associated  with the $(i-1) + 1$ to $(i-1) + p$-th rows of the matrix $A$. Then, it is easy to see that $A^{(i)}$ can be thought of as a randomly subsampled BOS with $p$ number of rows. 
 
	The following theorem provides a linear convergence of KZPT for this setting.
	
	\begin{thm}
		\label{thm:convergence_KZPT}
		Let $A$ be an $m \times N$ randomly subsampled bounded orthonormal matrix and we run KZPT with measurements $b = Ax +e$ for some $s$-sparse vector $x \in \R^N$. Suppose the row selection rule $\tau$ is a (possibly random) permutation of the index set $\{1,2, \dots, m\}$. Let $\delta_{(p)} \in (0,1/2)$. Assume that the period $p$ divides $m$. Then, the iterates of KZPT with $\gamma = N/p$ satisfy
	\begin{align*}
			  &\|x^{k+1}_1 - x\| \le (2\delta_{(p)})^{{m \over p} \cdot k} \|x^1_1 - x\| + {2 \over (1- 2\delta_{(p)}) \left( 1-(2\delta_{(p)})^{{m \over p}} \right) } \sup_{\substack{i \in [m/p] \\ \Omega \subset [N], |\Omega| \le 2s}} \left \|P_{\Omega} \left(  {1 \over p} \left( A^{(i)} \right)^Te\right) \right \|_2,
	\end{align*} with probability at least $1 - c{m \over pN^3}$ as long as $p \ge C \delta_{(p)}^{-2} s \ln^4 N$ for some universal constants $c, C > 0$. 

    Moreover, for the optimal choice of the period $p$, the iterates of KZPT satisfy
	\begin{align*}
	    &\|x^{k+1}_1 - x\| \le \min_{\substack{m \% p = 0 \\ p > 9C  s \ln^4 N}} \left[ 
         \left({\left(m \over p\right)}^{m \over 2p}  \left(  2 \sqrt{{C s \ln^4 N \over m}}   \right)^{m \over p} \right)^k \|x^1_1 - x\| + 18\sigma_p  \right],
    \end{align*}
    where $\sigma_p = \sup\limits_{\substack{i \in [m/p] \\ \Omega \subset [N], |\Omega| \le 2s}} \left \|P_{\Omega} \left(  {1 \over p} \left( A^{(i)} \right)^Te\right) \right \|_2$ with probability at least $1 - c{m \over N^3}$ as long as $m > 9C  s \ln^4 N$.
    \end{thm} 
	By comparing the convergence rate in Theorem \ref{thm:convergence_KZPT} with the one of KZIHT ($p =m$ case) in Theorem \ref{thm:KZIHT_convergence} for randomly subsampled BOS, KZPT is guaranteed to converge faster than KZIHT with a proper choice of the period $p$ in general.
	In particular, the following corollary offers more quantitative information on the optimal choice of the period $p$ to achieve the best possible convergence rate of KPZT for noiseless setting. 

	\begin{cor}
		\label{cor:KZPT_Corollary}
		Consider the noiseless setting, i.e., $e = 0$. We use the same notation and assumptions in Theorem \ref{thm:convergence_KZPT} and let $p$ be the closest integer to $C \cdot 4 e \cdot s \ln^4 N$ that divides $m$. Then the iterates of KZPT with this choice of period $p$ satisfies

		\[
			\|x^{k+1}_1 - x\| \le \left[	\exp \left(- {c m \over  s \ln^4 N} \right) \right]^k \|x^1_1 - x\|, 
		\] whereas the iterates of KZIHT obey
		\[
		 \|x^{k+1}_1 - x\| \le \left[ 2 \sqrt{{ C s \ln^4 N \over m}} \right]^k \|x^1_1 - x\|
		\] for some universal constants $c, C>0$.
	\end{cor}
	 
	\begin{proof} [Proof of Corollary \ref{cor:KZPT_Corollary}] 
        First, note that from the second part of Theorem \ref{thm:convergence_KZPT}, the iterates of KZIHT ($p = m$ case) obey
        \begin{align*}
	\|x^{k+1}_1 - x\| 
	\le \left( 2 \sqrt{{C s \ln^4 N \over m}}      \right)^k \|x^k_1 -x\|_2.
        \end{align*} 
		 Maximizing the rate in the theorem 
		\[
			\left[ {\left(m \over p\right)}^{m \over 2p}  \left(  2 \sqrt{{C s \ln^4 N \over m}}   \right)^{m \over p} \right]
		\] over $p$ with $p > 9C  s \ln^4 N$ and some algebraic manipulations give the rate of KZPT with optimal period $p$.
	\end{proof}
	
	\begin{proof} [Proof of Theorem \ref{thm:convergence_KZPT}] 
	From the observation in Section \ref{subsection:BOS}, one can easily obtain
	
	\[
	x^{k+1}_j = T_s \left( x^k_1 + \left({1 \over \sqrt{p}} A^{(i)} \right)^T \left( {1 \over \sqrt{p}} A^{(i)} \right)(x - x^k_j) + {1 \over p} \left( A^{(i)} \right)^Te \right),
	\] for $j$ divisible by $p$. 
	
    We denote the the RIP constants of the submatrices ${1 \over \sqrt{p}}A^{(i)}$ by $\delta_{s,p}^{(i)}$ for $s$-sparse vectors. 
    The assumption $p \ge C \delta_{(p)}^{-2} s \ln^4 N$ and
    Lemma \ref{lem:BOS_lemma} combined with the union bound argument imply that the RIP constants $\delta_{3s,p}^{(i)}$ for ${1 \over \sqrt{p}}A^{(i)}$ satisfy $\delta_{3s,p}^{(i)} \le \delta_{(p)}$,
     for all integers $1 \le i \le m/p$ with probability at least $1 - c{m \over pN^3}$ for some universal constant $c > 0$. 
    
	Note that when $p = m$, we have the upper bound of the RIP constant for the sampling matrix ${1 \over \sqrt{m}} A$, which is $\delta_{3s,m} \le \delta_{(m)}$.

    From the convergence result of KZIHT in Theorem \ref{thm:KZIHT_convergence} and the above bound for the RIP constant $\delta_{3s,m}$ we have
    \begin{align*}
	&\|x^{k+1}_1 - x\| \le \left(2\delta_{(m)} \right)^k \|x^1_1 - x\| +  \cdot {2 \over 1 - 2 \delta_{(m)}} \cdot \sup_{\Omega \subset [N], |\Omega| \le 2s} \left \|P_{\Omega} \left( {1 \over m}  A^Te \right) \right \|_2.
    \end{align*} 

    On the other hand, because the thresholding operator is applied for $m/p$ times for each epoch in KZPT,  the iterates of KZPT satisfy
	\begin{align*}
	    &\|x^{k+1}_1 - x\| 
	    \le (2\delta_{(p)})^{m \over p} \|x^k_1 - x\| + {2 \over 1- 2\delta_{(p)}} \sup_{\substack{i \in [m/p] \\ \Omega \subset [N], |\Omega| \le 2s}} \left \|P_{\Omega} \left(  {1 \over p} \left( A^{(i)} \right)^Te\right) \right \|_2. 
	\end{align*}
	Thus, by induction on $k$, we have
    \begin{align*}
         &\|x^{k+1}_1 - x\| \\
	    &\le (2\delta_{(p)})^{{m \over p} \cdot k} \|x^1_1 - x\| + {2 \over (1- 2\delta_{(p)}) \left( 1-(2\delta_{(p)})^{{m \over p}} \right) } \sup_{\substack{i \in [m/p] \\ \Omega \subset [N], |\Omega| \le 2s}} \left \|P_{\Omega} \left(  {1 \over p} \left( A^{(i)} \right)^Te\right) \right \|_2.
    \end{align*}
    This shows the first part of the theorem. 
    
    To prove the second part, set 
    $\delta_{(p)} = \sqrt{C s \ln^4 N \over p} < {1 \over 3}$ for $p > 9C  s \ln^4 N$ and $m \%p = 0$. Then, by plugging $\delta_{(p)} = \sqrt{C s \ln^4 N \over p}$ to the conclusion of the first part of the theorem, we have
    \begin{align*}
         &\|x^{k+1}_1 - x\| \le  \left( 2 \sqrt{{ C s \ln^4 N \over p}} \right)^{{m \over p} \cdot k} \|x^1_1 - x\| + 18  \sup_{\substack{i \in [m/p] \\ \Omega \subset [N], |\Omega| \le 2s}} \left \|P_{\Omega} \left(  {1 \over p} \left( A^{(i)} \right)^Te\right) \right \|_2.
    \end{align*}

    Since the period $p$ is a hyperparameter for KZPT as long as $ p > 9C  s \ln^4 N$, the error bound for KZPT iterates can be rewritten as 
    \small
    \begin{align*}
	& \|x^{k+1}_1 - x\|  \le \min_{\substack{m \% p = 0 \\ p > 9C  s \ln^4 N}} \left[\left( 2 \sqrt{{ C s \ln^4 N \over p}} \right)^{{m \over p} \cdot k} \|x^1_1 - x\| + 18  \sup_{\substack{i \in [m/p] \\ \Omega \subset [N], |\Omega| \le 2s}} \left \|P_{\Omega} \left(  {1 \over p} \left( A^{(i)} \right)^Te\right) \right \|_2 \right].    
    \end{align*}
    \normalsize
    After noticing $ 2 \sqrt{{C s \ln^4 N \over m}} = \sqrt{p \over m} \times 2 \sqrt{{ C s \ln^4 N \over p}}$, the KZPT error bound can be also given by
    \begin{align*}
	&\|x^{k+1}_1 - x\| \le \min_{\substack{m \% p = 0 \\ p > 9C  s \ln^4 N}}  \left[ 
         \left({\left(m \over p\right)}^{m \over 2p}  \left(  2 \sqrt{{C s \ln^4 N \over m}}   \right)^{m \over p} \right)^k \|x^1_1 - x\| + 18 \sigma_p    \right],
    \end{align*}
    where $\sigma_p = \sup\limits_{\substack{i \in [m/p] \\ \Omega \subset [N], |\Omega| \le 2s}} \left \|P_{\Omega} \left(  {1 \over p} \left( A^{(i)} \right)^Te\right) \right \|_2$. This proves the second part of the theorem. 
    \end{proof}
	
	\section{Numerical Experiments}
	\label{section:Numerical Experiments}
	In this section, we present numerical experiments to confirm our theory and demonstrate the effectiveness of KZPT and KZIHT. Throughout the experiments, the row selection rule for KZPT/KZIHT is the random reshuffling by default unless stated otherwise. The step size for KZIHT is set as $\gamma = N/m$ throughout the experiments as specified in our theorems, unless stated otherwise. 
	
	For each trial in the experiments, we create an $s$-sparse vector  $x^*$ in $\R^N$   by first selecting $s$ components  uniformly at random among $[N]$, assigning $s$ values from the i.i.d. standard normal distribution, and zeroing out the rest of the components. The performance of the methods in each test is evaluated by the relative error, ${\|x - x^*\| \over \|x^*\|}$ either in the number of epochs (the number of outer iterations) or wall-clock time. We run $30$ trials for each test and the relative error averaged over $30$ trials is recorded in the plots. The code used for our numerical experiments is available at \cite{code_KZIHT}. 
	
	\subsection{KZ and KZIHT}
	We create a $1024 \times 512$ Bernoulli random matrix and randomly generate a unit vector $x^*$ with various sparsity levels $5, 10, 15, 20$. In each test, we run the original KZ with a random shuffling rule and record its relative error versus the number of KZIHT iterations. We compare the error decay curves of the original randomized KZ and KZIHT for sparsity level $s$ in the number of epochs and wall-clock time.  Note that KZIHT outperforms KZ for the overdetermined case as shown in Figure \ref{fig:Commparison_KZ_KZIHT}.  As for the underdetermined setting, KZ doesn't converge to a sparse solution ($N > m$) in general \cite{ma2015convergence}, so we include the tests only for the overcomplete setting ($m \ge N$).
	
	\begin{figure} 
		\centering
		\begin{minipage}{0.45\textwidth}
			\centering
			\includegraphics[width=1.00 \textwidth]{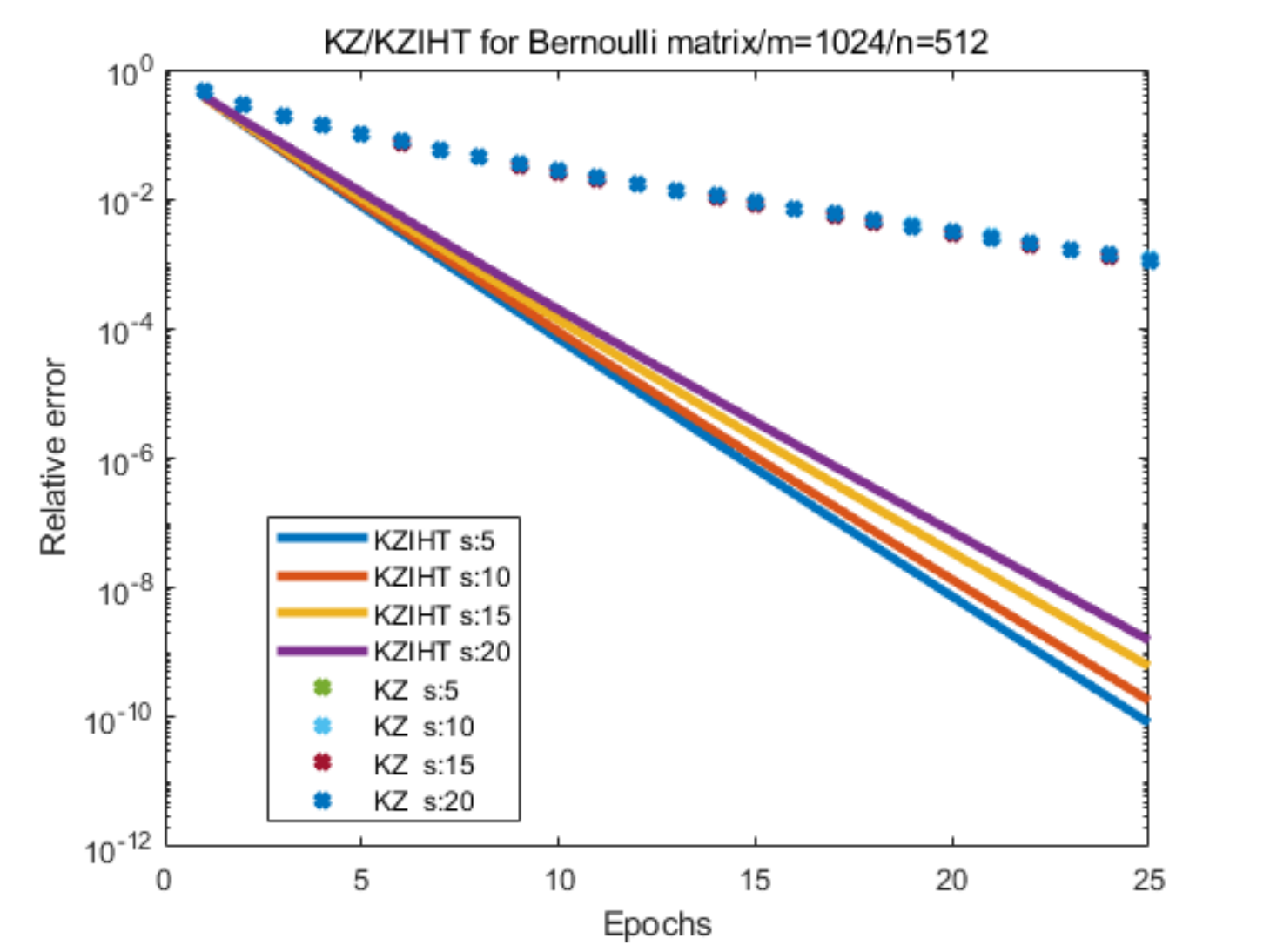}
		\end{minipage}\hfill
		\begin{minipage}{0.45\textwidth}
			\centering
			\includegraphics[width=1.00 \textwidth]{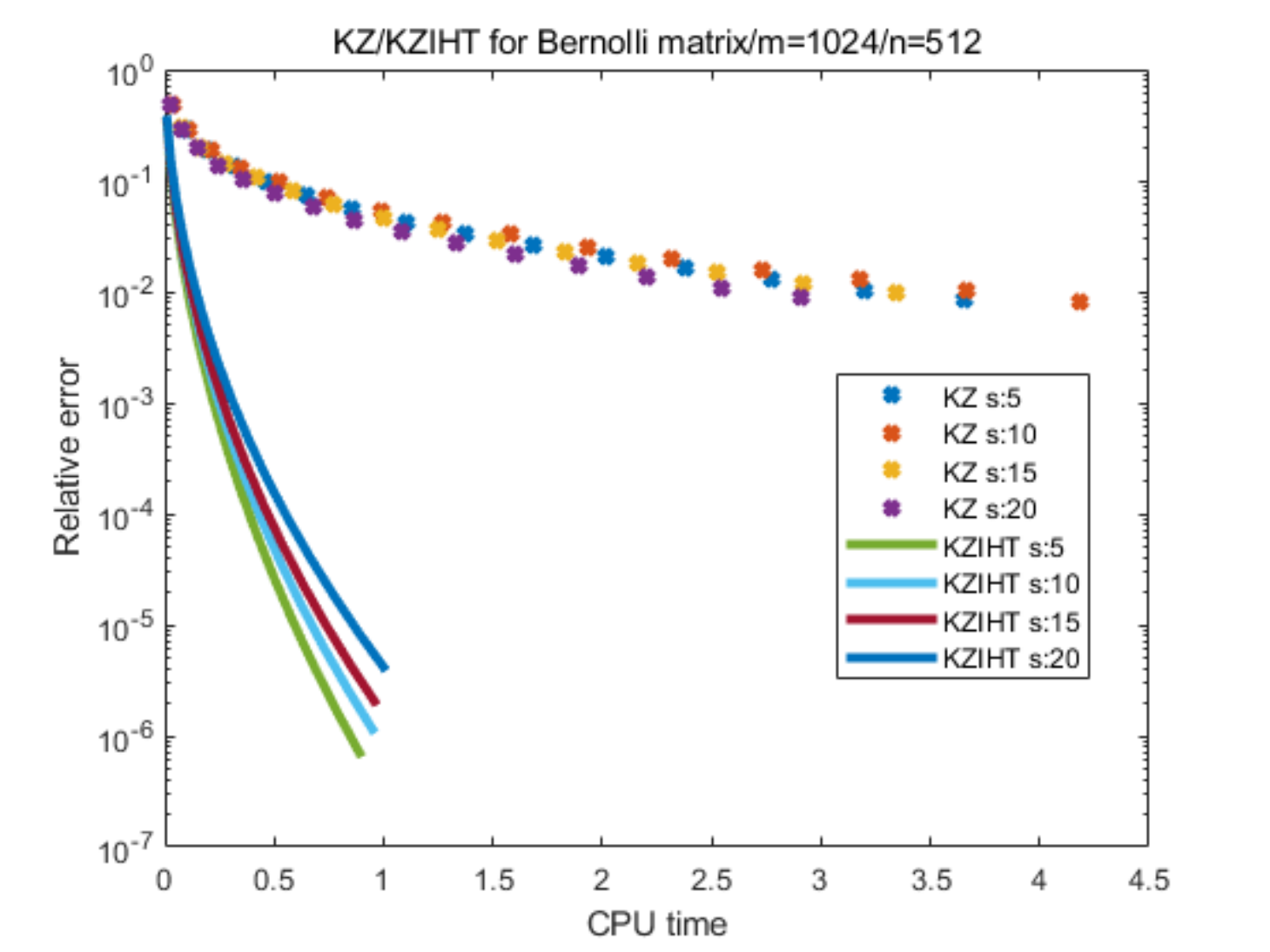}
		\end{minipage}
		\caption[Relative error of KZ/KZIHT] {Comparison between the relative error of KZ with random shuffling and KZIHT in epochs and wall-clock time for various sparsity levels. KZIHT outperforms KZ in both criteria. 
		}
		\label{fig:Commparison_KZ_KZIHT}
	\end{figure}
	
	\subsection{IHT and KZIHT}
	In this experiment, we compare IHT and KZIHT. The plots in Figure \ref{fig:Commparison_IHT_KZIHT_epoch}  indicate that KZIHT can handle higher sparsity levels than IHT for Gaussian and Bernoulli sensing matrices, where the number of measurements is $m=800$, $n=1024$, and $s$ is the sparsity level. These two plots show that IHT converges only when $s=50$ at about the same rate as KZIHT, and for higher sparsity levels $100$, $150$, $200$ IHT diverges whereas KZIHT still converges linearly. 
	
	In order to compare IHT and KZIHT in more extensive settings, we include the recovery phase transition plots in Figure \ref{fig:PhaseTransition_IHT_KZIHT_Bernoulli}. The success probabilities of the signal recovery of IHT and KZIHT for the Bernoulli sensing matrix are recorded in the plots.
	The horizontal axes of the plots correspond to the number of nonzero components of the signal (sparsity level $s$) and the vertical axes correspond to the number of measurements $m$. The colormap represents the recovery success probability over $30$ trials in the experiment. 
	 Figure \ref{fig:PhaseTransition_IHT_KZIHT_Bernoulli} shows that KZIHT outperforms IHT for almost all the settings in the experiment. We obtain similar plots for the Gaussian case, so they are not included here. 
	
	\begin{figure} 
		\centering
		\begin{minipage}{0.45\textwidth}
			\centering
			\includegraphics[width=1.00 \textwidth]{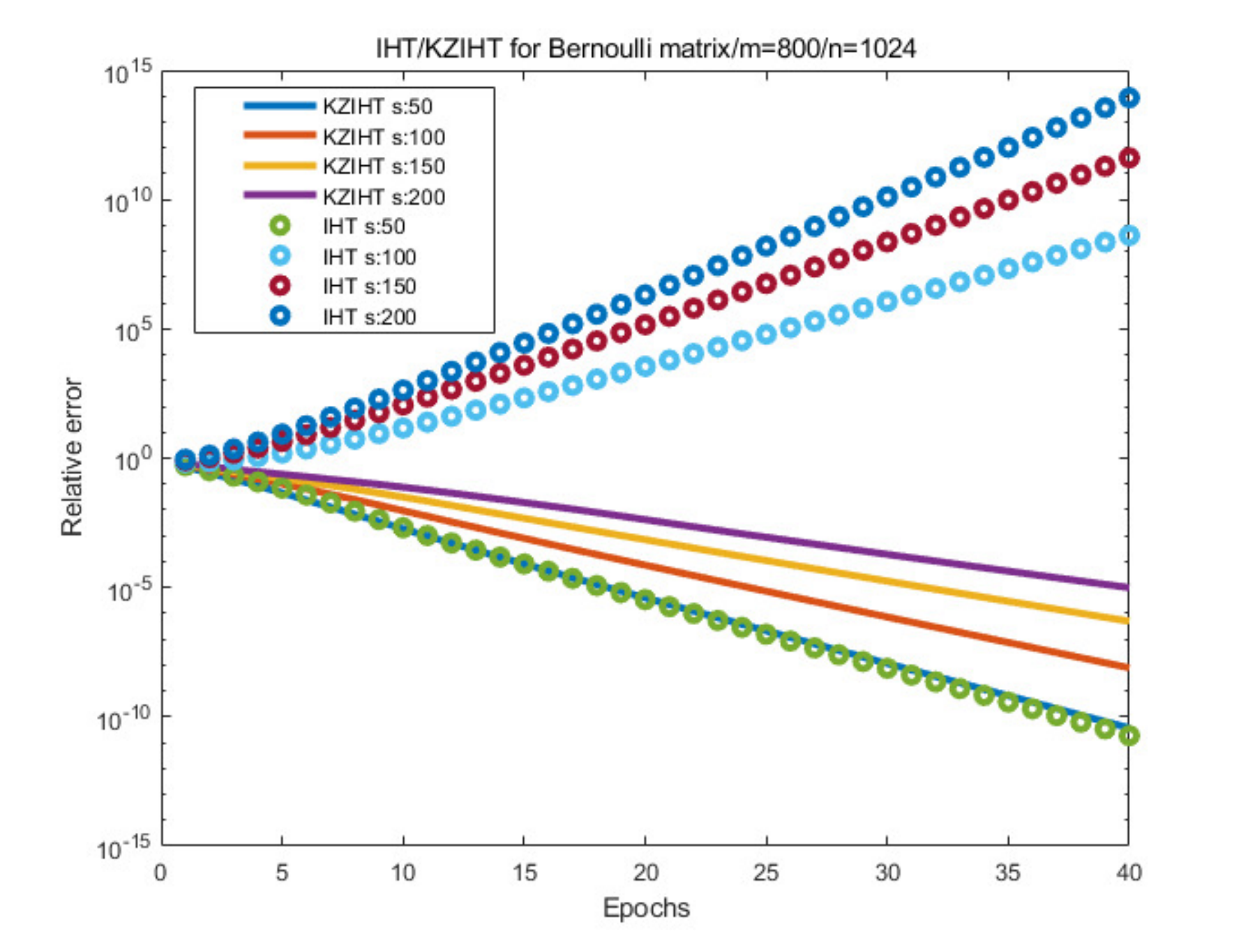}
		\end{minipage}\hfill
		\begin{minipage}{0.45\textwidth}
			\centering
			\includegraphics[width=1.00 \textwidth]{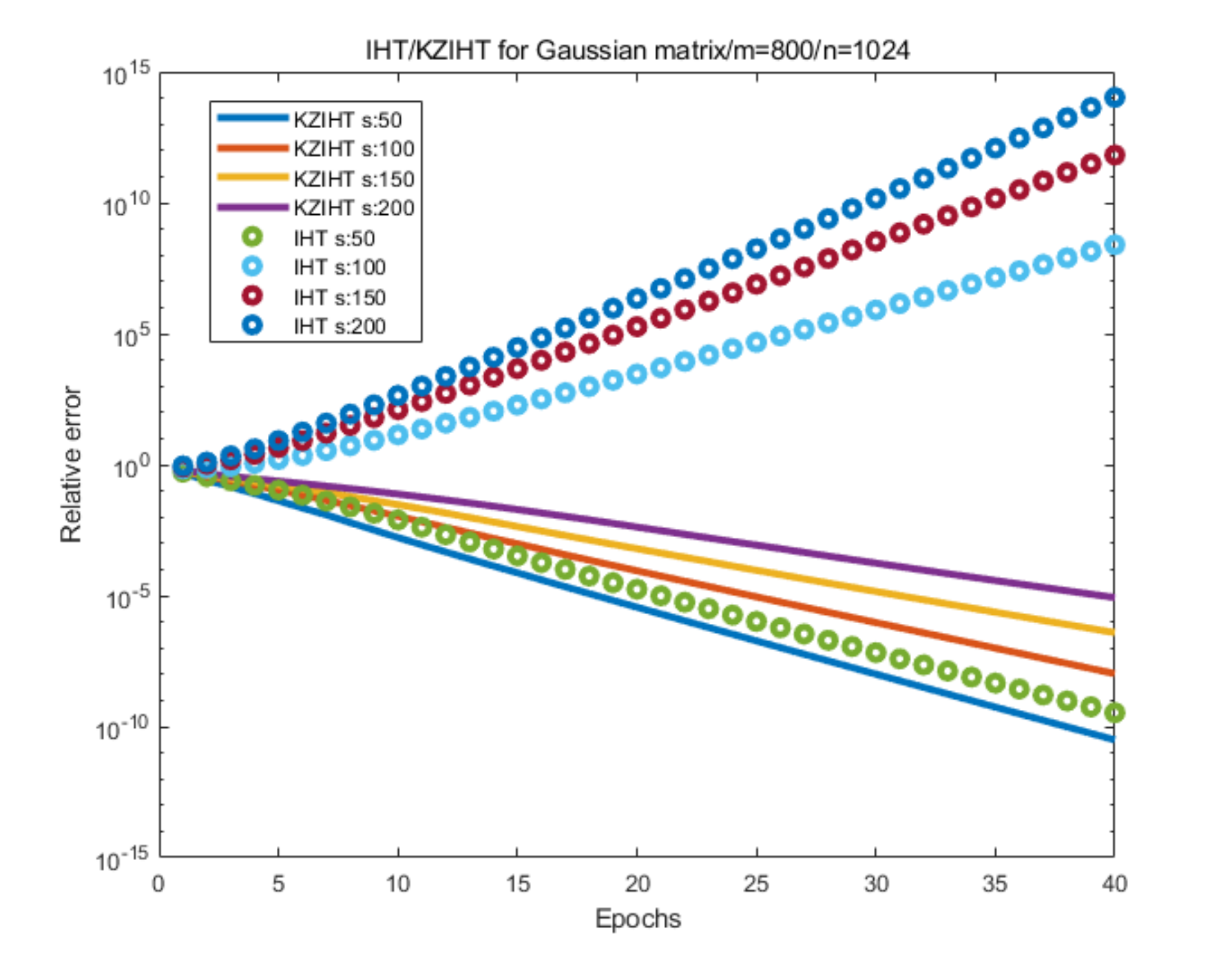}
		\end{minipage}
		\caption[Relative error of IHT/KZIHT] {Comparison between the relative error of IHT and KZIHT in epochs for various sparsity levels in  Bernoulli and Gaussian sensing matrix cases.
		}
		\label{fig:Commparison_IHT_KZIHT_epoch}
	\end{figure}

	\begin{figure}
		\centering
		\begin{minipage}{0.45\textwidth}
			\centering
			\includegraphics[width=1.00 \textwidth]{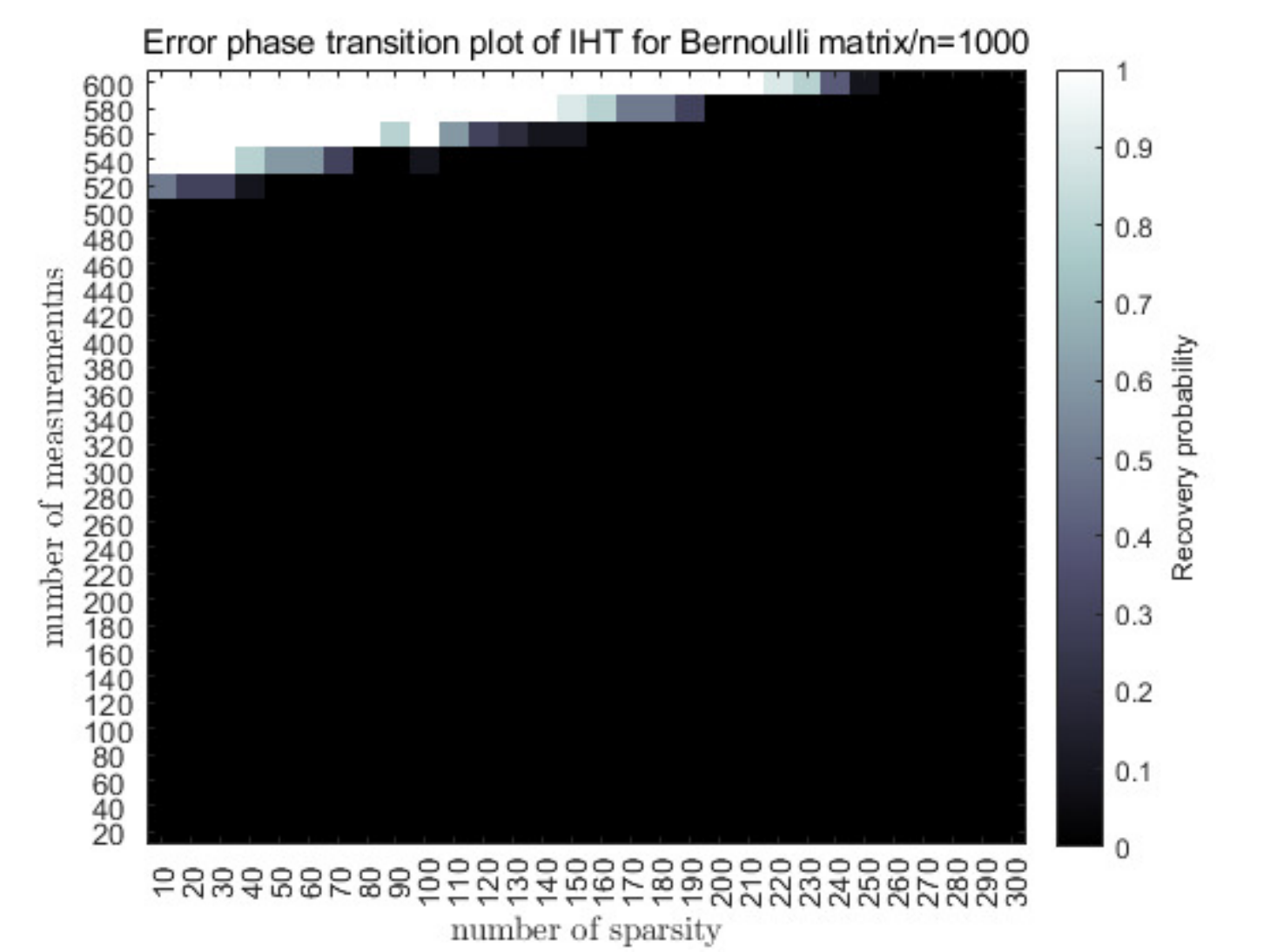}
		\end{minipage}\hfill
		\begin{minipage}{0.45\textwidth}
			\centering
			\includegraphics[width=1.00 \textwidth]{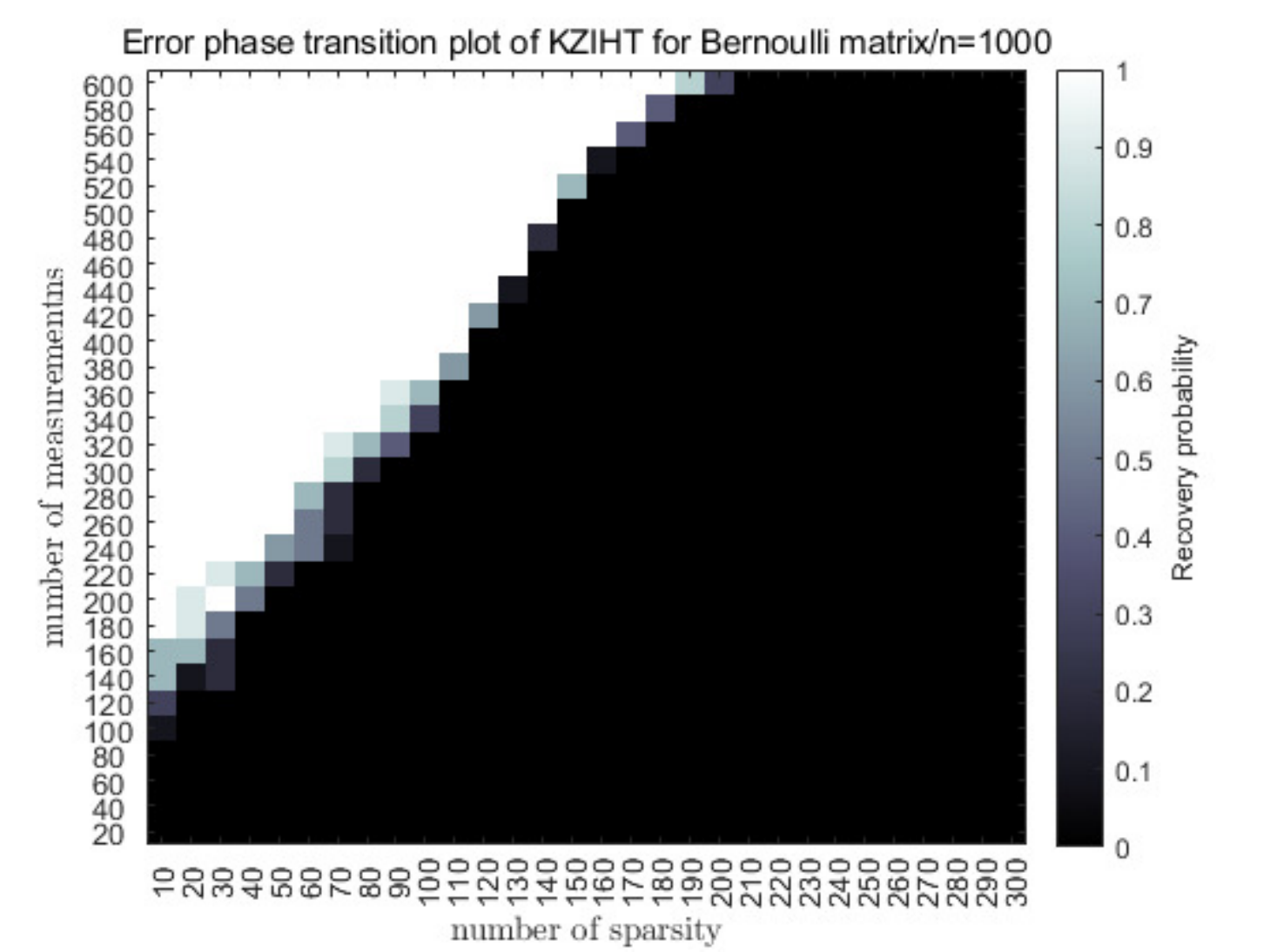}
		\end{minipage}
		\caption[Phase Transition of IHT/KZIHT for Bernoulli] {Recovery phase transition plots of IHT and KZIHT for Bernoulli sensing matrix. We record the success probabilities of recovery of sparse signals in $\R^{1000}$ with error below $10^{-1}$ for different sparsity levels and the number of measurements. KZIHT outperforms IHT for almost all the settings in the experiment.
		}
		\label{fig:PhaseTransition_IHT_KZIHT_Bernoulli}
	\end{figure}
	
	As for the randomly subsampled Hardmard matrix case, Figure \ref{fig:Commparison_IHT_KZIHT} illustrates a performance comparison of IHT and KZIHT. Here the number of measurements is $m=256$, the signal dimension is $n=1024$, and we vary the sparsity level $s$ over the set $\{5, 10, 15, 20\}$. The plots suggest that their performances are comparable in terms of the error decay rate in epoch, but IHT converges slightly faster than KZIHT in wall-clock time. 
	
		\begin{figure}
		\centering
		\begin{minipage}{0.45\textwidth}
			\centering
			\includegraphics[width=1.00 \textwidth]{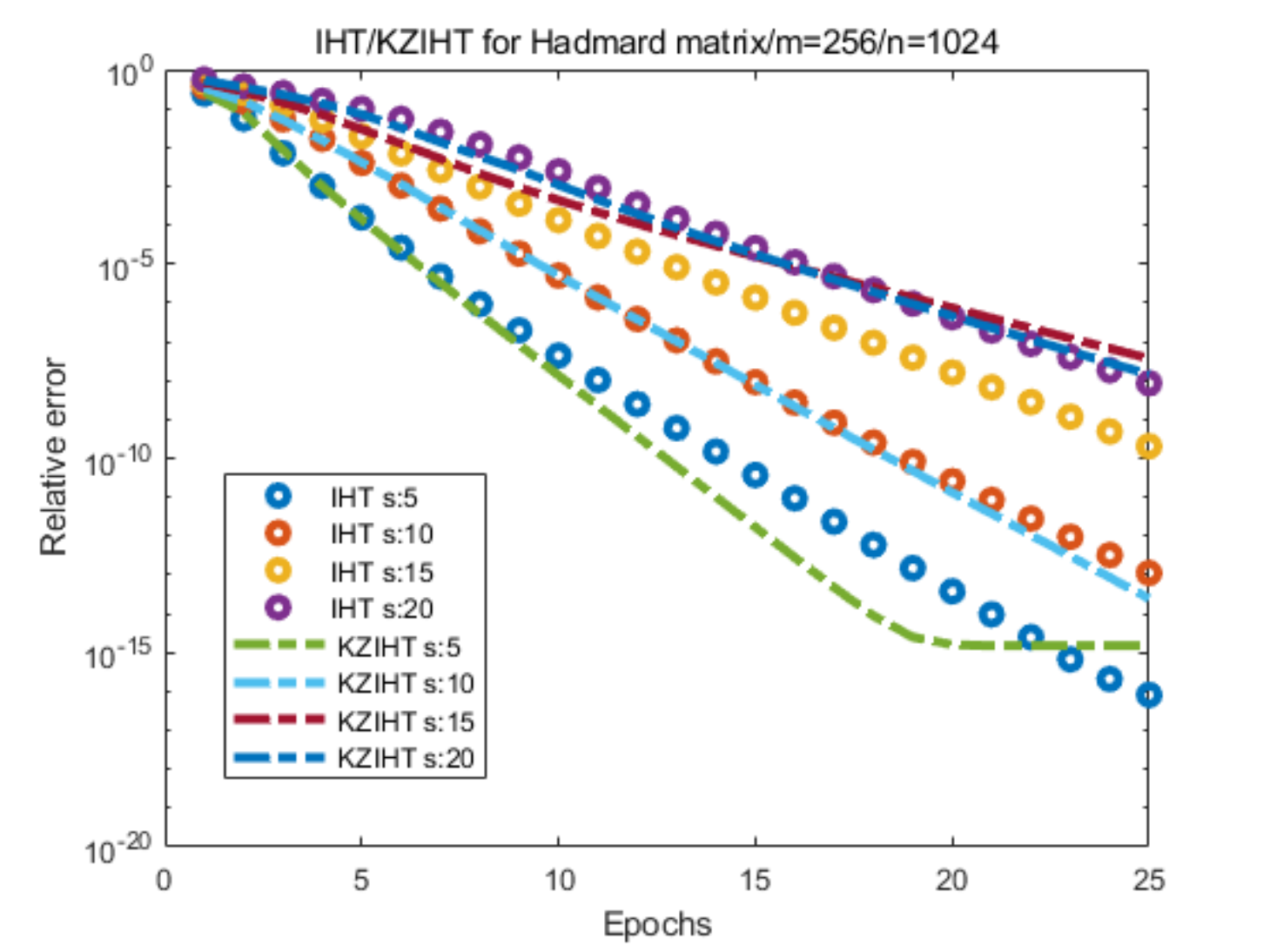}
		\end{minipage}\hfill
		\begin{minipage}{0.45\textwidth}
			\centering
			\includegraphics[width=1.00 \textwidth]{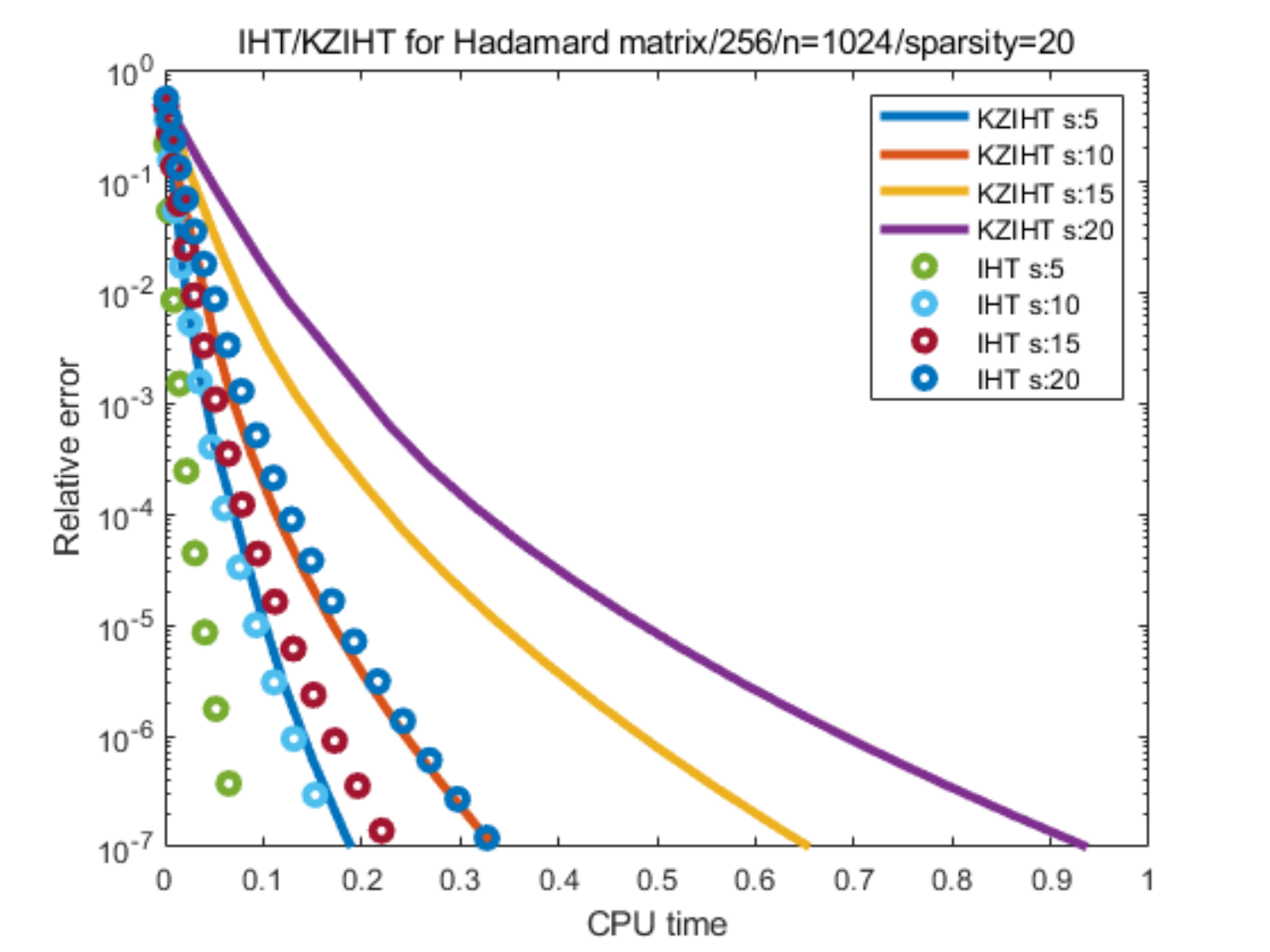}
		\end{minipage}
		\caption[Relative error of IHT/KZIHT] {Comparison between the relative error of IHT and KZIHT in epochs and wall-clock time for various sparsity levels in Hadamard sensing matrix case. KZIHT and IHT exhibit similar convergence rates in the number of outer iterations (epochs). IHT converges faster than KZIHT in the wall-clock time. 
		}
		\label{fig:Commparison_IHT_KZIHT}
	\end{figure}
	
	Figure \ref{fig:PhaseTransition_IHT_KZIHT_Hadamard} shows the  recovery phase transition plots of IHT and KZIHT for the Hadamard matrix case. This is consistent with our theory (Theorem \ref{thm:KZIHT_convergence} and its proof) for the randomly subsampled BOS case, indicating that IHT and KZIHT should perform the same. 
	
	\begin{figure} 
		\centering
		\begin{minipage}{0.45\textwidth}
			\centering
			\includegraphics[width=1.00 \textwidth]{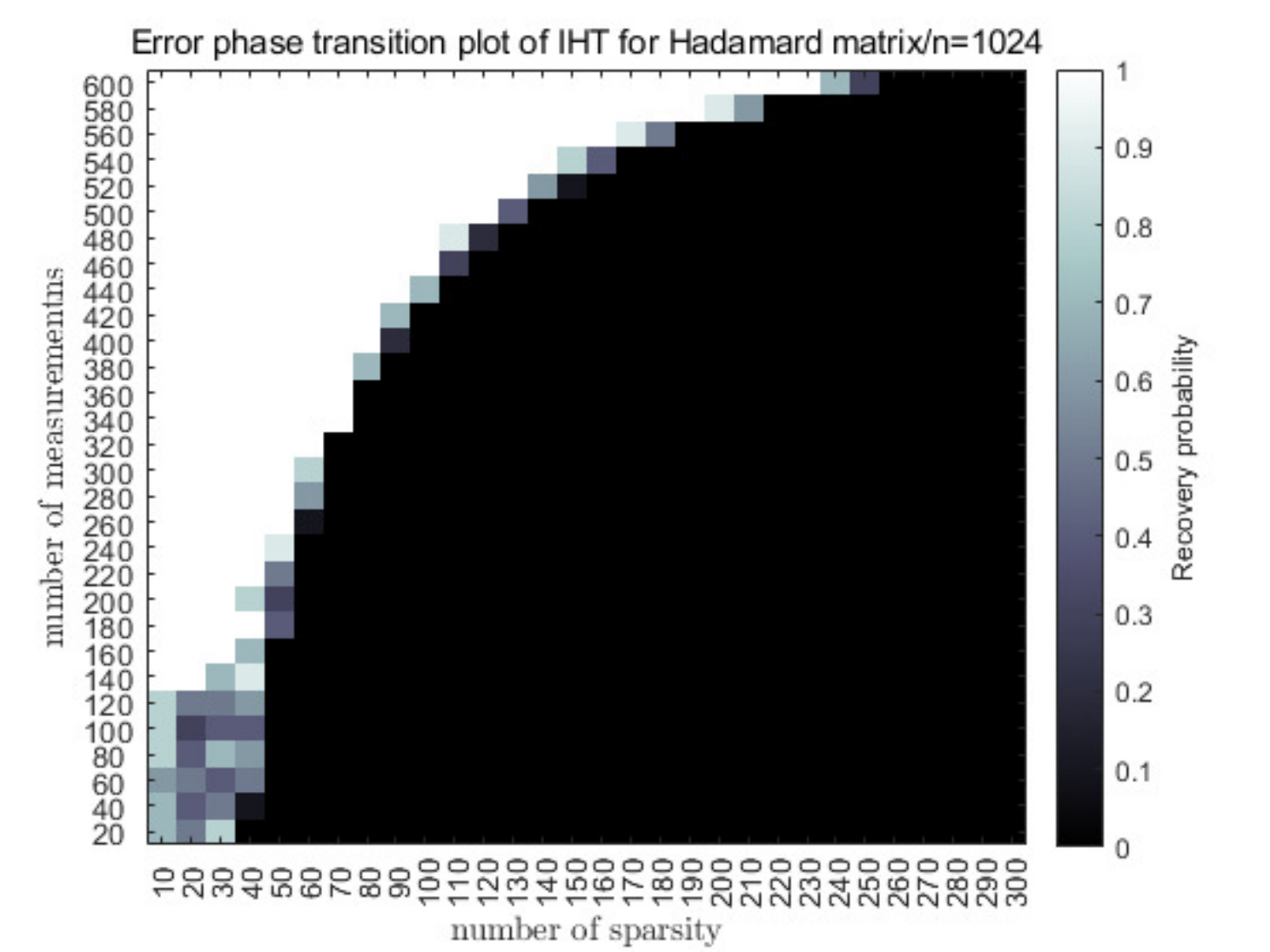}
		\end{minipage}\hfill
		\begin{minipage}{0.45\textwidth}
			\centering
			\includegraphics[width=1.00 \textwidth]{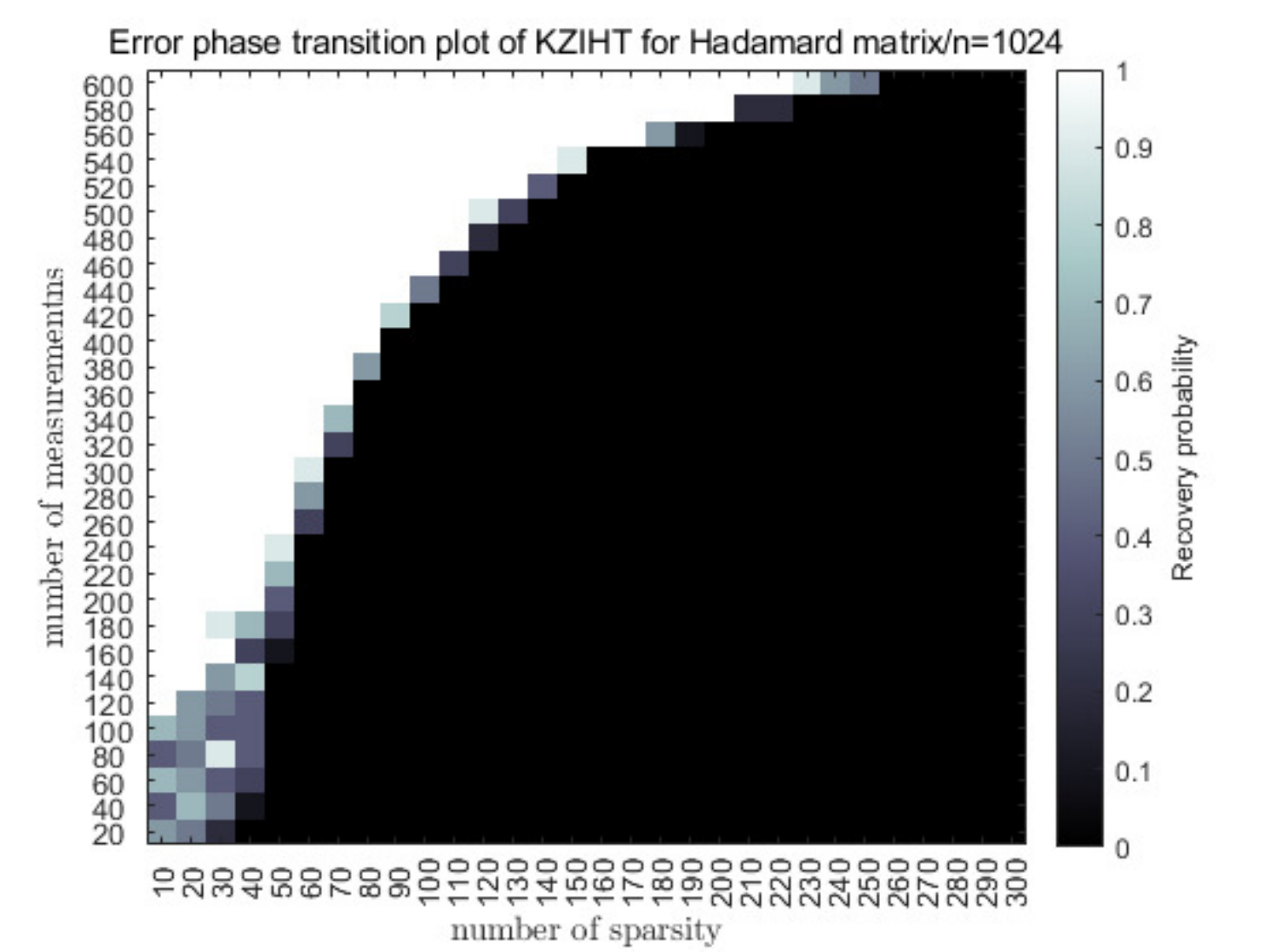}
		\end{minipage}
		\caption[Phase Transition of IHT/KZIHT for Hadamard] {Recovery phase transition plots of IHT and KZIHT for Hadamard sensing matrix. We record the success probabilities of recovery of sparse signals in $\R^{1000}$ with error below $10^{-1}$ for different sparsity levels and the number of measurements. The phase transition plots of IHT and KZIHT performs are almost identical.
		}
		\label{fig:PhaseTransition_IHT_KZIHT_Hadamard}
	\end{figure}
	
	Our numerical results show that KZIHT has practical advantages over IHT, especially in high sparsity regimes for Bernoulli/Gaussian sensing matrices, which are the sub-Gaussian matrices satisfying the assumptions in Theorem \ref{thm:main_theorem2}.  In the following sections, we will numerically demonstrate that KZPT with a nontrivial period $p$ improves the performance of KZIHT including the Hadamard sensing matrix case as well.
 
	\subsection{IHT and KZPT}
	We use $256 \times 1024$ randomly subsampled Hardmard as the sampling matrix, so $N = 1024$ and $m = 256$. 
	Figure \ref{fig:Commparison_IHT_KZPT} shows that KZPT with $\gamma = N/m$, $\lambda = 1$, period $80$ converges for all five different sparsity levels, whereas IHT doesn't converge at all. This may be due to the fact that the convergence rate of IHT is the same as KZIHT for randomly subsampled BOS and Theorem \ref{thm:convergence_KZPT} implies that KZPT with proper choice of period converges faster than KZIHT. 
	
	\begin{figure} 
    	\centering
    	\begin{minipage}{0.45\textwidth}
    		\centering
    		\includegraphics[width=1.00 \textwidth]{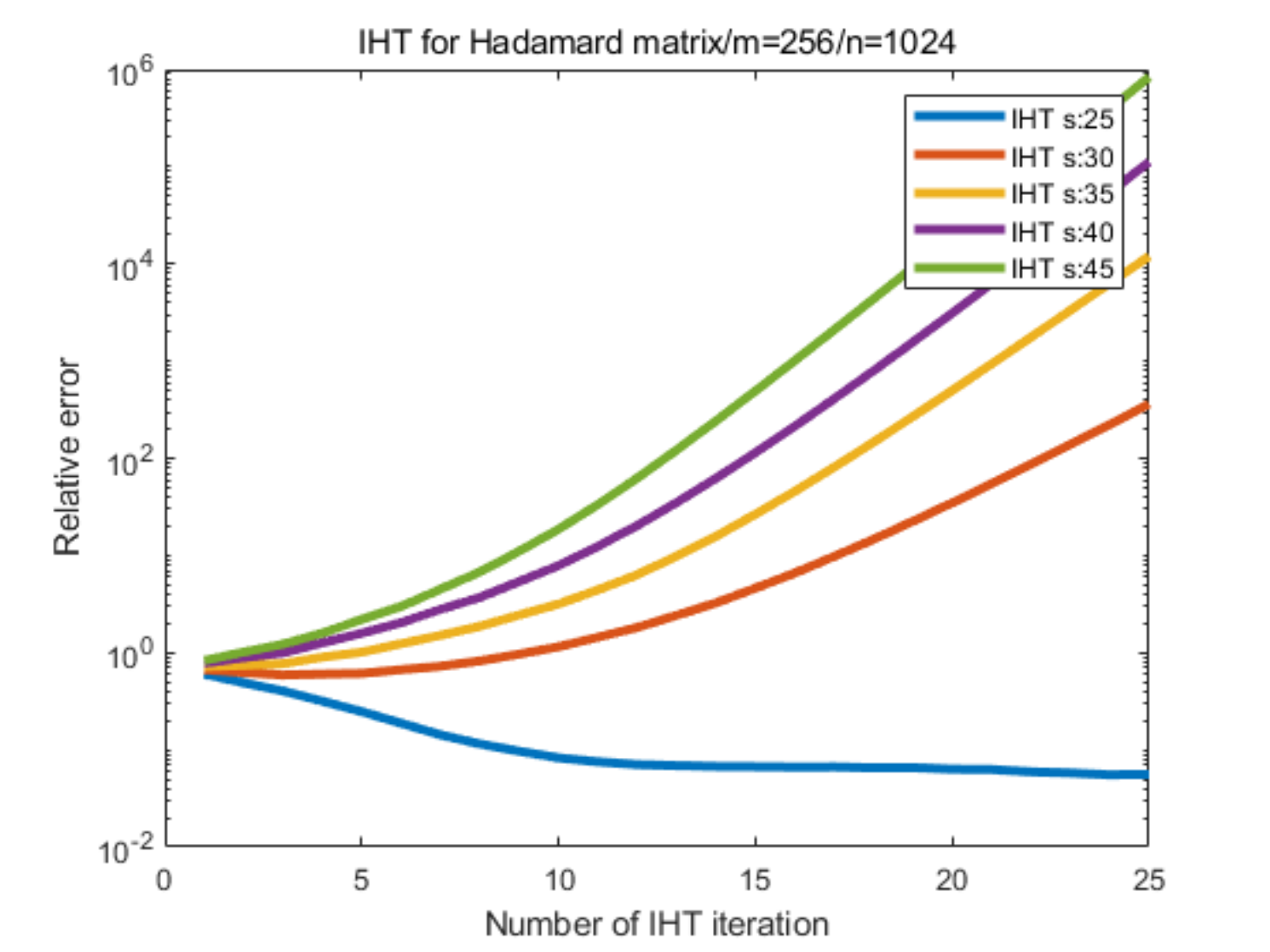}
    	\end{minipage}\hfill
    	\begin{minipage}{0.45\textwidth}
    		\centering
    		\includegraphics[width=1.00 \textwidth]{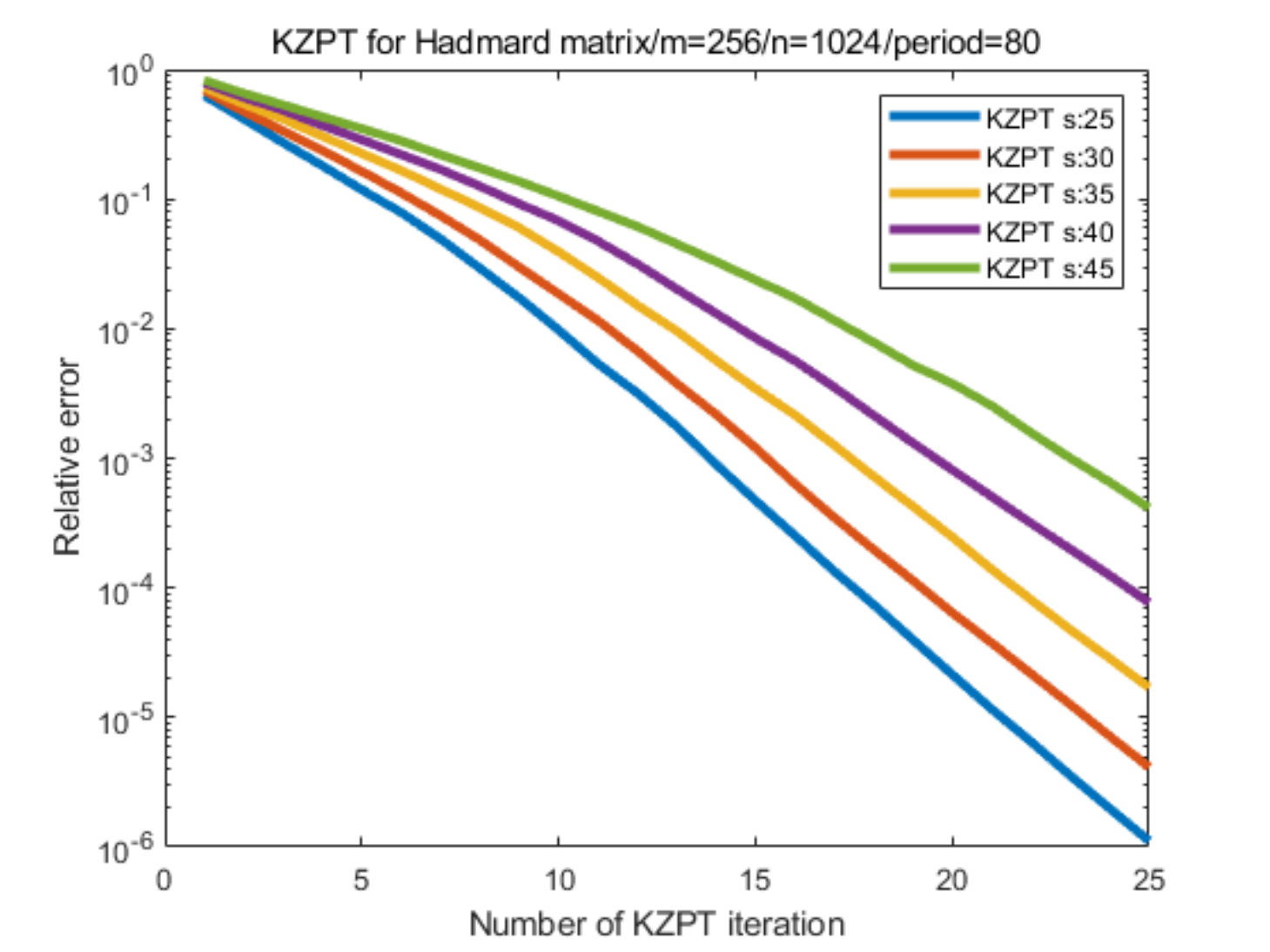}
    	\end{minipage}
    	\caption[Relative error of IHT/KZPT] {Comparison between the relative error of IHT and KZPT in the number of iterations for various sparsity levels. IHT does not converge (diverges except for sparsity level $25$) whereas KZPT with period $80$ still converges linearly. 
    	}
    	\label{fig:Commparison_IHT_KZPT}
    \end{figure}
	
	\subsection{KZIHT and KZPT}
	In this set of experiments, we compare KZIHT and KZPT. The first experiment is conducted with $N = 512$ and $m = 256$. The left panel of Figure \ref{fig:Commparison_KZIHT_KZPT} illustrates that KZPT with period $p = 128$, $\gamma = N/p$, $\lambda = 1$ converges faster than IHT and KZIHT in the number of epochs for randomly subsampled Hardmard sampling matrix.  Similar plots are obtained for a $256 \times 1024$ randomly subsampled Hardmard sampling matrix ($N = 1024$ and $m = 256$), which is displayed in the right panel of Figure \ref{fig:Commparison_KZIHT_KZPT}. Since the sparsity levels for the numerical experiments in Figure \ref{fig:Commparison_KZIHT_KZPT}  are relatively small compared to $m$, which makes the RIP constant bounds small as well, the discussion after Theorem \ref{thm:convergence_KZPT} explains why KZPT outperforms KZIHT. Note that in these experiments, we did not optimize the period $p$, so the best performance of KZPT could be even better.

	\begin{figure} 
		\centering
		\begin{minipage}{0.45\textwidth}
			\centering
			\includegraphics[width=1.00 \textwidth]{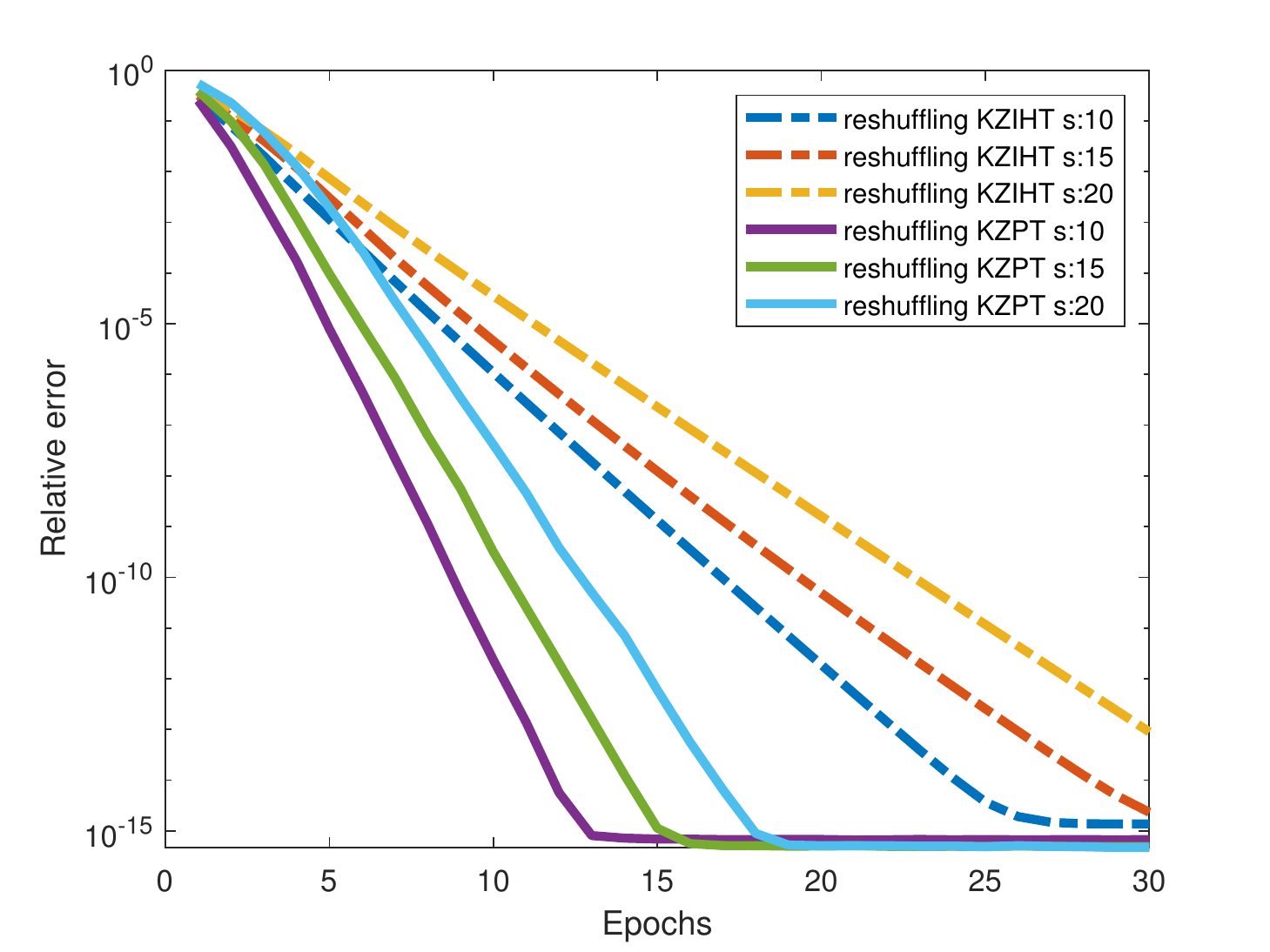}
		\end{minipage}\hfill
		\begin{minipage}{0.45\textwidth}
			\centering
			\includegraphics[width=1.00 \textwidth]{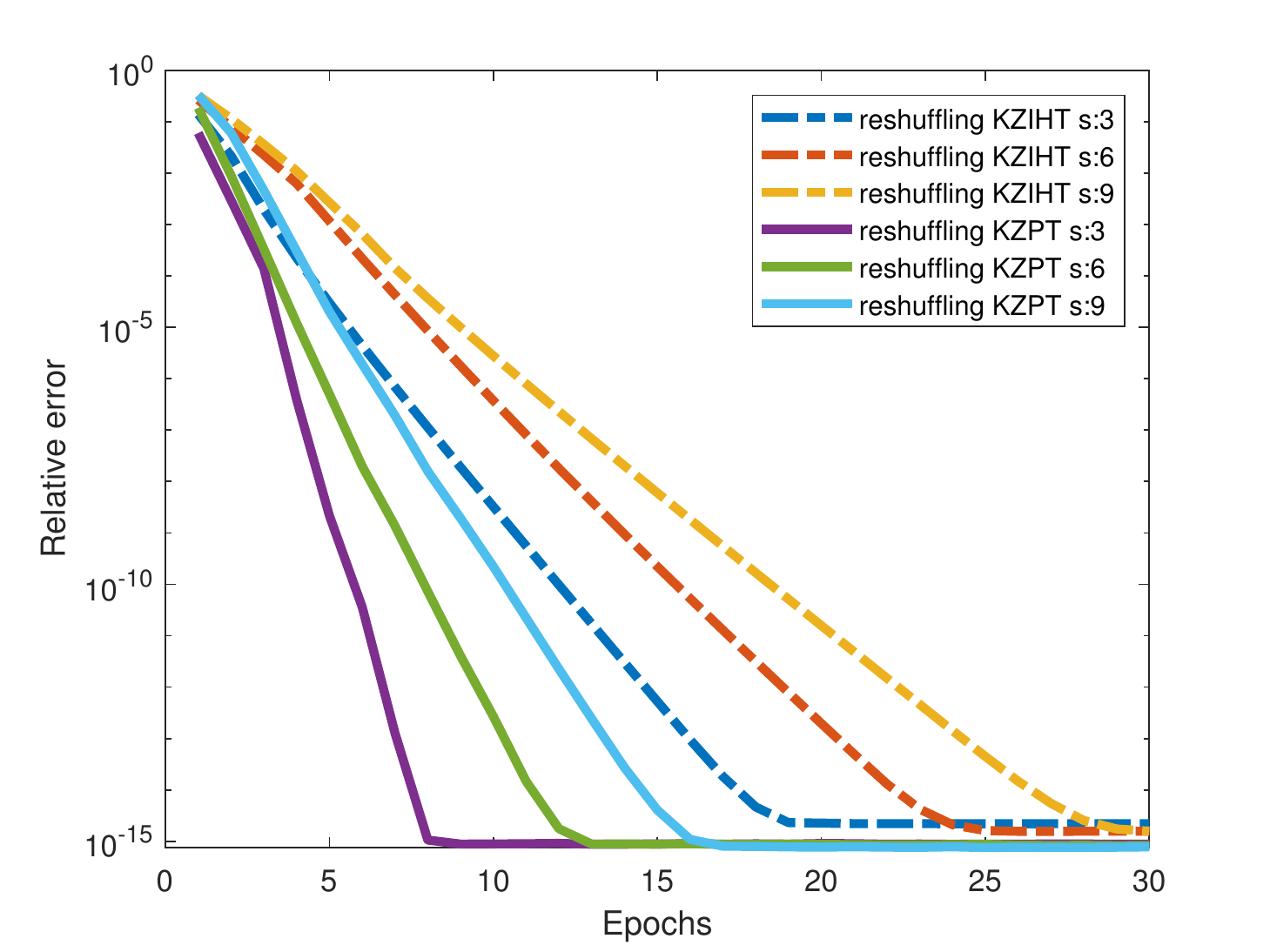}
		\end{minipage}
		\caption[Relative error of KZIHT/KZPT] {Comparison between the relative error of KZIHT and KZPT with period 128 in the number of iterations for various sparsity levels.  KZPT converges faster than KZIHT. 
		}
	\label{fig:Commparison_KZIHT_KZPT}
	\end{figure}

	Recall that in Figure \ref{fig:Commparison_IHT_KZPT}, we have empirically observed that KZPT converges for the Hadamard matrix case while IHT doesn't converge for any sparsity levels in the settings. As KZIHT can only handle the same sparsity levels as IHT for the Hadamard case, we would have obtained the same plots in the left panel of Figure \ref{fig:Commparison_IHT_KZPT} for KZIHT, so this also implies that KZPT performs better than KZIHT. Thus, our numerical experiments evidence the benefits of using KZPT with proper choice of the period $p$ over IHT or KZIHT.  
	\subsection{Reshuffling KZIHT and 
Randomized KZIHT}
	In this section, we compare the reshuffling KZIHT and randomized KZIHT, the original version of KZIHT proposed in \cite{zhang2015iterative}. The row selection rule for randomized KZIHT is sampling with replacement. Recall that we propose to use and analyze reshuffling schemes for KZIHT/KZPT throughout the paper, which is based on sampling without replacement. 
    In the experiment, we use a randomly subsampled Hadamard matrix. The number of measurements is $m=256$, the signal dimension is $n=1024$, and the sparsity levels are $5$ and $10$. In the experiments for Figure \ref{fig:Commparison_KZIHT_sampling}, the step size $\gamma = 1$ (original Kaczmarz setting) is used for the left plots and $\gamma = n/m = 4$ is used for the right. 
    As we observe in Figure \ref{fig:Commparison_KZIHT_sampling}, reshuffling KZIHT converges faster than randomized KZIHT for both settings in our experiments. The right plot shows that the randomized KZIHT actually diverges for the larger step size $\gamma = 4$. This indicates that the reshuffling scheme is superior to KZIHT with sampling with replacement, which is consistent with recent findings about the advantages of reshuffling methods over the traditional SGD for various scenarios \cite{ahn2020sgd, mishchenko2020random, sun2020optimization, gurbuzbalaban2021random}. 
	
	\begin{figure} 
		\centering
		\begin{minipage}{0.47\textwidth}
			\centering
			\includegraphics[width=1 \textwidth]{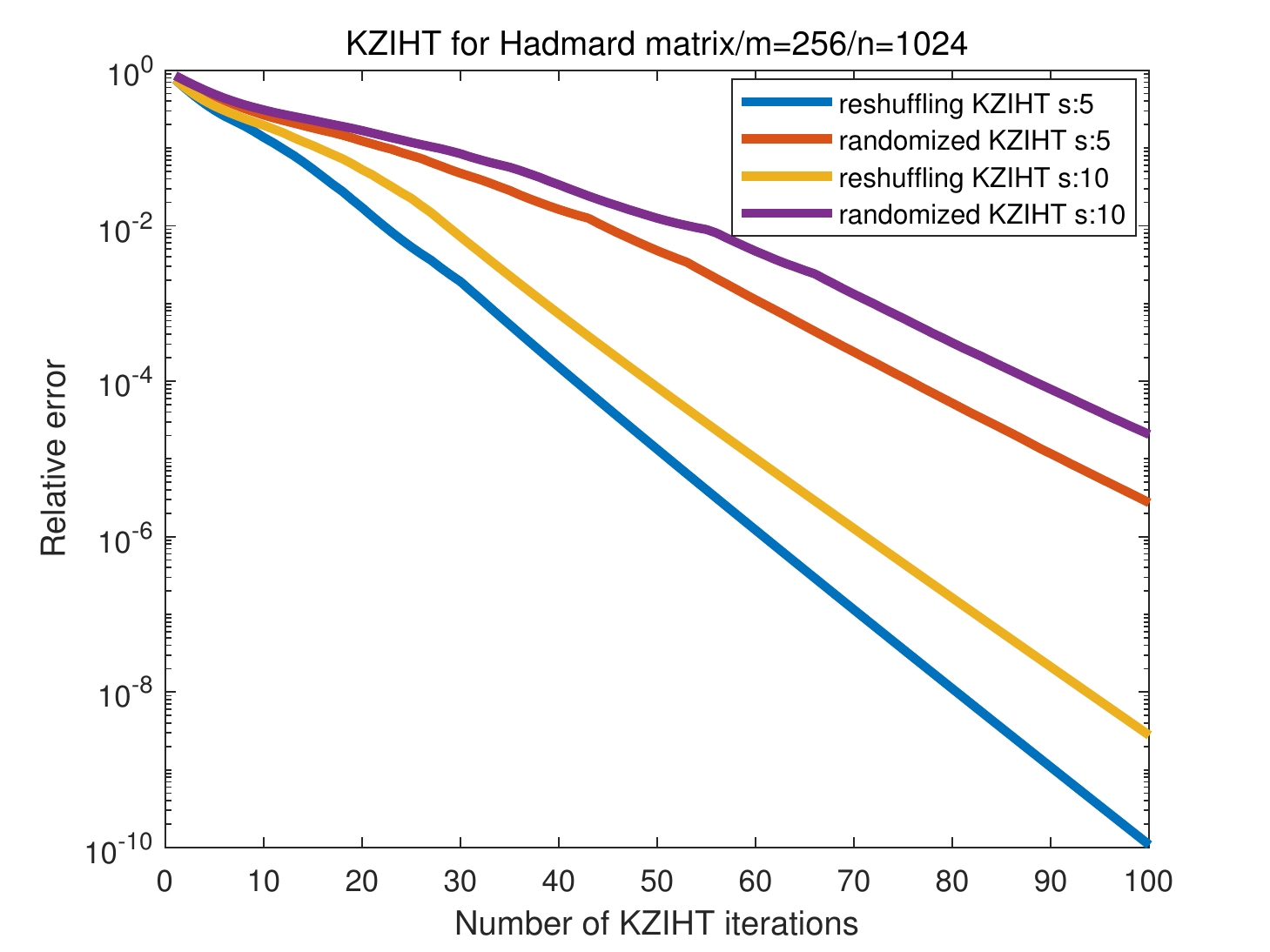}
		\end{minipage}
		\begin{minipage}{0.47\textwidth}
			\centering
			\includegraphics[width=1 \textwidth]{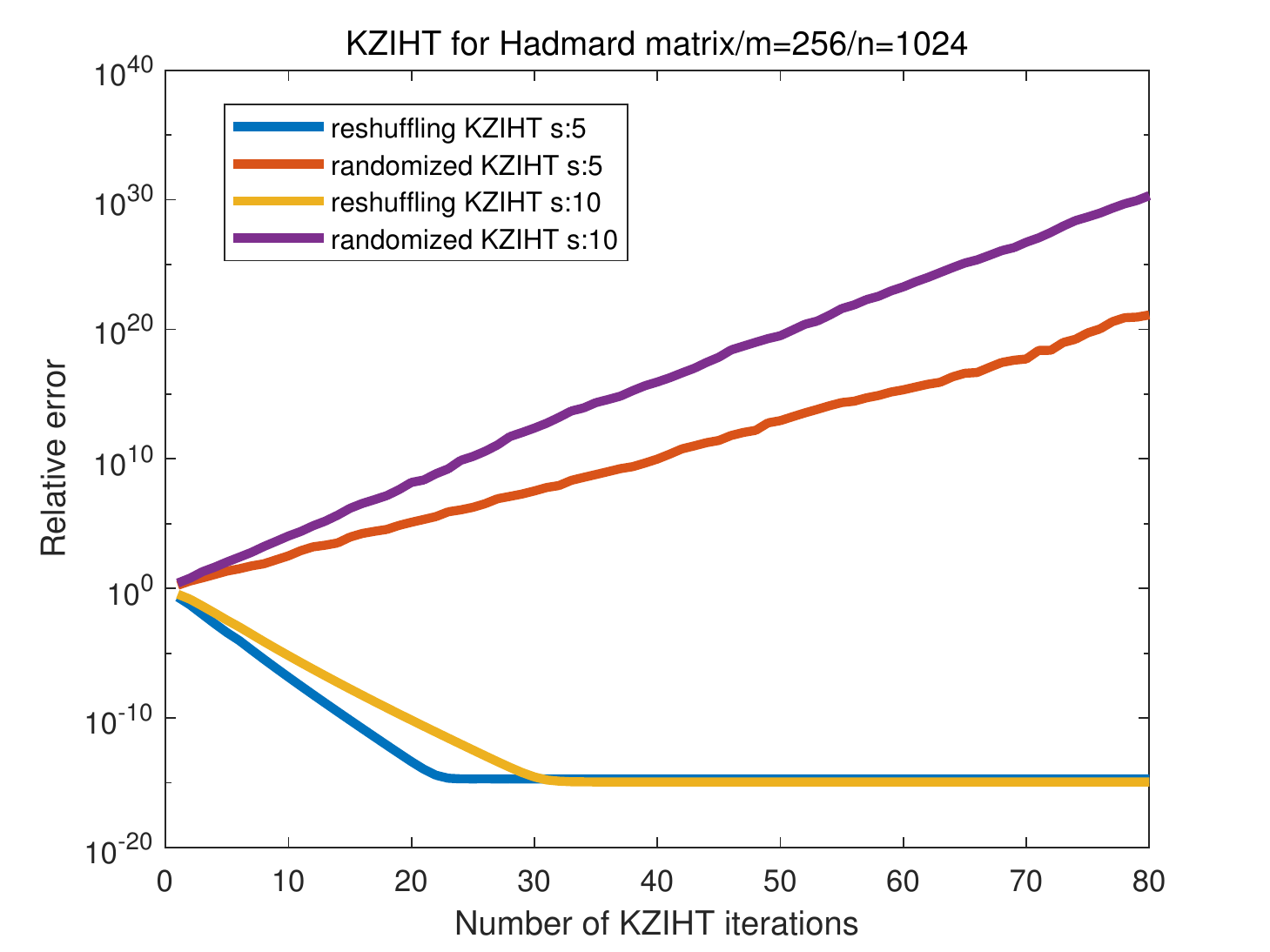}
		\end{minipage}
		\caption[Relative error of KZIHT/KZPT] {Comparison between the relative error curves of reshuffling KZIHT and randomized KZIHT in the number of iterations for various sparsity levels.  
		}
		\label{fig:Commparison_KZIHT_sampling}
	\end{figure}

\section{Additional Related Work}
Nuniti et al. \cite{nutini2016convergence} have studied the convergence of RK based on multi-step analysis, but they have focused on its application to greedy row selection rules to improve the performance of RK, rather than compressed sensing or tresholding-based methods. On the other hand, we analyze KZ-based thresholding methods with shuffling row selection rules to obtain sparse solutions to (underdetermined) linear systems by utilizing the near orthogonal property of random matrices. 

In order to show theoretical guarantees of KZIHT or KZPT for sub-Gaussian random matrices, we rely on the estimates of the matrix product, 
\[
\left(I - \gamma {a_m a_m^T \over \|a_m\|_2^2} \right) \left(I - \gamma {a_{m-1} a_{m-1}^T \over \|a_{m-1}\|_2^2} \right) \cdots \left(I - \gamma {a_1 a_1^T \over \|a_1\|_2^2} \right).
\]
There are interesting recent works about the matrix product of this form, for example, \cite{emme2017limit, huang2021matrix, lai2020recht, recht2012toward} giving concentration bounds for the matrix product. However, these bounds still seem to be too loose to achieve our goal  for mainly two reasons. First, the effective number of terms in the product such that their bounds are applicable should be at least $\Omega(N)$, which is generally not applicable to the compressed sensing setting or underdetermined systems with $m \ll N$. Second, many analyses do not take into account the signs of the terms in the product, so in fact they only offer the same bounds for the product $ \left(I + \gamma {a_m a_m^T \over \|a_m\|_2^2} \right) \left(I + \gamma {a_{m-1} a_{m-1}^T \over \|a_{m-1}\|_2^2} \right) \cdots \left(I + \gamma {a_1 a_1^T \over \|a_1\|_2^2} \right)$, which are not meaningful for our purpose since they would only give convergence rates greater than $1$ in our analysis. 

	
\section{Conclusions}
This paper gives the first theoretical analysis of the convergence of KZIHT and proposes our new method KZPT, which generalizes KZIHT in two ways by introducing periodic thresholding and step sizes. We have shown that KZIHT enjoys linear convergence up to optimal statistical bias for randomly subsampled BOS. Extending our idea for the proof of the KZIHT convergence analysis enables us to show that KZPT converges linearly up to optimal statistical bias for sub-Gaussian random matrices as well. Additionally, we have derived the optimal choice of the thresholding period for the best performance of KZPT for the noiseless case, which indicates that KZPT outperforms KZIHT if the period and step sizes are chosen properly. Various numerical experiments have been conducted to complement our theory and demonstrate the potential advantages of our proposed methods.

\bibliographystyle{plain}
\bibliography{KZPT_Final}
	
\end{document}